%% file: OnLearningSetsOfSymmetricElements.tex
\documentclass{article}

\usepackage{microtype}
\usepackage{graphicx}
\usepackage{subfigure}
\usepackage{booktabs} 
\usepackage{multicol}
\usepackage{multirow}
\usepackage{times}
\usepackage{amsmath}
\usepackage{amssymb}
\usepackage{wrapfig}
\usepackage{natbib}

\usepackage{xr-hyper}
\usepackage{hyperref}

\newcommand{\DSS}{Deep Sets for Symmetric elements }
\newcommand{\ignore}[1]{}

\input{math_commands.tex}

\usepackage[accepted]{icml2020}

\icmltitlerunning{On Learning Sets of Symmetric Elements}

\begin{document}
\twocolumn[
\icmltitle{On Learning Sets of Symmetric Elements}



\icmlsetsymbol{equal}{*}

\begin{icmlauthorlist}
\icmlauthor{Haggai Maron}{nr}
\icmlauthor{Or Litany}{st}
\icmlauthor{Gal Chechik}{ba,nr}
\icmlauthor{Ethan Fetaya}{ba}
\end{icmlauthorlist}

\icmlaffiliation{nr}{NVIDIA Research}
\icmlaffiliation{ba}{Bar Ilan University}
\icmlaffiliation{st}{Stanford University}

\icmlcorrespondingauthor{Haggai Maron}{hmaron@nvidia.com}

\icmlkeywords{Machine Learning, ICML}

\vskip 0.3in
]



\printAffiliationsAndNotice{}  
\begin{abstract}
  Learning from unordered sets is a fundamental learning setup, recently attracting increasing attention. Research in this area has focused on the case where elements of the set are represented by feature vectors, and far less emphasis has been given to the common case where set elements themselves adhere to their own symmetries. That case is relevant to numerous applications, from deblurring image bursts to multi-view 3D shape recognition and reconstruction. 
  In this paper, we present a principled approach to learning sets of general symmetric elements. We first characterize the space of linear layers that are equivariant both to element reordering and to the inherent symmetries of elements, like translation in the case of images. We further show that networks that are composed of these layers, called \textit{\DSS{}} layers (DSS), are universal approximators of both invariant and equivariant functions, \revision{and that these networks are strictly more expressive than Siamese networks}. DSS layers are also straightforward to implement. Finally, we show that they improve over existing set-learning architectures in a series of experiments with images, graphs and point-clouds.
  
 
\end{abstract}


\section{Introduction}\label{s:intro}
Learning with data that consists of unordered sets of elements is an important problem with numerous applications, from classification and segmentation of 3D data ~\cite{zaheer2017deep, qi2017pointnet, su2015multi, kalogerakis20173d} to image deblurring \cite{aittala2018burst}. In this setting, each data point consists of a set of elements, and the task is independent of element order. This independence induces a symmetry structure, which can be used to design deep models with improved efficiency and generalization. Indeed, models that respect set symmetries, e.g. \cite{zaheer2017deep, qi2017pointnet}, have become the leading approach for solving such tasks. However, in many cases, the elements of the set themselves adhere to certain symmetries, as happens when learning with sets of images, sets of point-clouds and sets of graphs. It is still unknown what is the best way to utilize these additional symmetries. 

A common approach to handle per-element symmetries, is based on processing elements individually. First, one  processes each set-element independently into a feature vector using a Siamese architecture \cite{bromley1994signature}, and only then fuses information across all feature vectors. When following this process, the interaction between the elements of the set only occurs after each element has already been processed, possibly omitting low-level details. Indeed, it has been recently shown that for learning sets of images \cite{aittala2018burst, sridhar2019multiview, liu2019permutation}, significant gain can be achieved with intermediate information-sharing layers.

In this paper, we present a principled approach to learning sets of symmetric elements. First, we describe the symmetry group of these sets, and then fully characterize the space of linear layers that are equivariant to this group. Notably, this characterization implies that information between set elements should be shared in all layers. For example, Figure \ref{fig:layer} illustrates a DSS layer for sets of images.  DSS layers provide a unified framework that generalizes several previously-described architectures for a variety of data types. In particular, it directly generalizes DeepSets \cite{zaheer2017deep}. Moreover, other recent works can also be viewed as special cases of our approach \cite{Hartford2018, aittala2018burst, sridhar2019multiview}. 

A potential concern with equivariant architectures is that restricting layers to be equivariant to some group of symmetries may reduce the expressive power of the model \cite{maron2019universality, morris2018weisfeiler, xu2018how}. We eliminate this potential limitation by proving two universal-approximation theorems for invariant and equivariant DSS networks. Simply put, these theorems state that if invariant (equivariant) networks for the elements of interest are universal, then the corresponding invariant (equivariant) DSS networks on sets of such elements are also universal. \revision{One important corollary of these results is that DSS networks are strictly more expressive than Siamese networks}.



To summarize, this paper has three main contributions: (1) We characterize the space of linear equivariant layers for sets of elements with symmetries. (2) We prove two universal approximation theorems for networks that are composed of DSS layers. (3) We demonstrate the empirical benefits of the DSS layers in a series of tasks, from classification through matching to selection, applied to diverse data from images to graphs and 3D point-clouds.  These experiments show consistent improvement over previous approaches. 

\ignore{
}

\begin{figure}[t]
    \centering
    \includegraphics[width=1.05\linewidth]{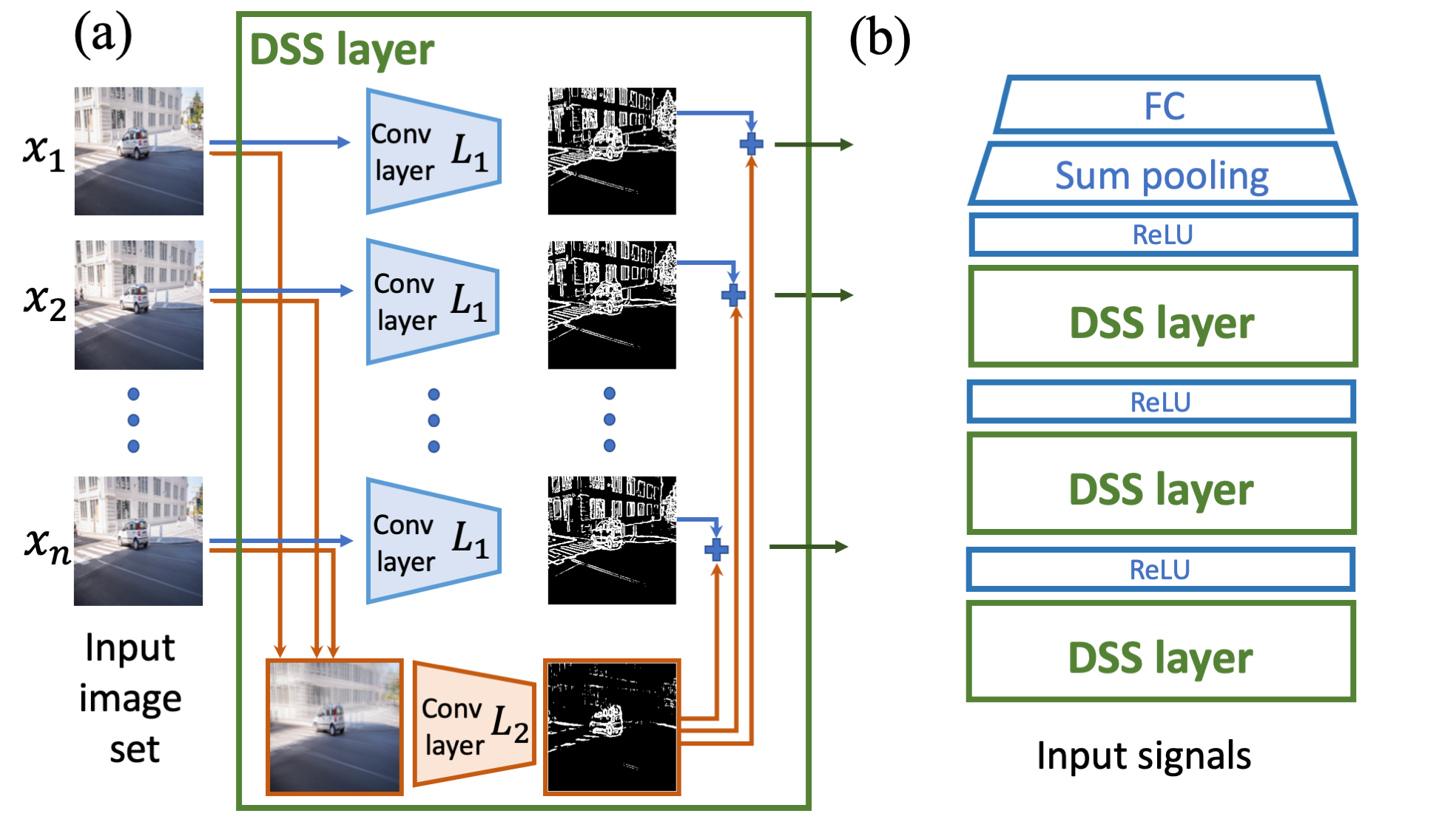}
    \caption{(a) A DSS layer for a set of images is composed of Siamese layer (blue) and an aggregation module (orange). The Siamese part is a convolutional layer ($L_1$) that is applied to each element independently. In the aggregation module, the \emph{sum} of all images is processed by a different convolutional layer ($L_2$) and is added to the output of the Siamese part. (b) An example of a simple DSS-based invariant network.} 
    \label{fig:layer}
\end{figure}

\section{Previous work}\label{s:prev}

\paragraph{Learning with sets.}


Several studies designed network architectures for set-structured input. \citet{vinyals2015order} suggested to extend the sequence-to-sequence framework of \citet{sutskever2014sequence} to handle sets. The prominent works of \citet{ravanbakhsh2016deep, edwards2016towards, zaheer2017deep, qi2017pointnet} proposed to use standard feed-forward neural networks whose layers are constrained to be equivariant to permutations. These models, when combined with a set-pooling layer, were also shown to be universal approximators of continuous permutation-invariant functions.  \citet{wagstaff2019limitations} provided a theoretical study on the limitations of representing functions on sets with such networks. In another related work, \citet{murphy2018janossy} suggested to model permutation-invariant functions as an average of permutation-sensitive functions. 


The specific case of learning sets of images was explored in several studies. \citet{su2015multi, kalogerakis20173d} targeted classification and segmentation of 3D models by processing images rendered from several view points. These methods use a Siamese convolutional neural network to process the images, followed by view-pooling layer. \citet{esteves2019equivariant} recently considered the same setup and suggested to perform convolutions on a subgroup of the rotation group, which enables joint processing of all views. \citet{sridhar2019multiview} tackled 3D shape reconstruction from multiple view points and suggest using several equivariant mean-removal layers in which the mean of all images is subtracted from each image in the set. \citet{aittala2018burst} targeted image burst deblurring and denoising, and suggested to use set-pooling layers after convolutional blocks in which for each pixel, the maximum over all images is concatenated to all images. \citet{liu2019permutation} proposed to use an attention-based information sharing block for face recognition tasks. In \citet{Gordon2020Convolutional} the authors modify neural processes by adding a translation equivariance assumption, treating the inputs as a set of translation equivariant objects.

\paragraph{Equivariance in deep learning.} The prototypical example for equivariance in learning is probably visual object recognition, where the prevailing Convolutional Neural Networks (CNNs) are constructed from convolution layers which are equivariant to image translations. In the past few years, researchers have used invariance and equivariance considerations to devise deep learning architectures for other types of data. In addition to set-structured data discussed above, researchers suggested equivariant models for interaction between sets \cite{Hartford2018}, graphs \cite{Kondor2018, maron2018invariant, maron2019provably, chen2019equivalence, albooyeh2019incidence} and relational databases \cite{graham2019deep}. Another successful line of work took into account other image symmetries such as reflections and rotations \cite{dieleman2016exploiting, cohen2016group, Cohen2016, worrall2017harmonic,cheng2018rotdcf}, spherical symmetries \cite{Welling2018, cohen2019gauge, esteves20173d}, or 3D symmetries \cite{Weiler2018,winkels20183d,worrall2018cubenet,kondor2018n,thomas2018tensor,Weiler2018}. From a theoretical point of view, several papers studied the properties of equivariant layers \cite{Ravanbakhsh2017, Kondor2018a, cohen2019general} and characterized the expressive power of models that use such layers \cite{yarotsky2018universal, maron2019universality, keriven2019universal,maehara2019simple,segol2019universal}.

\section{Preliminaries}\label{s:prem}
\subsection{Notation and basic definitions}
Let $x\in \R^\ell$ represent an input that adheres to a group of symmetries $G \leq S_\ell$, the symmetric group on $\ell$ elements. $G$  captures those transformations that our task-of-interest is invariant (or equivariant) to. The action of $G$ on $\R^\ell$ is defined by $(g\cdot x)_i=x_{g^{-1}(i)}$.  For example, when inputs are images of size $h\times w$, we have $\ell=hw$ and $G$ can be a group that applies cyclic translations, or left-right reflections to an image. A function is called $G$-equivariant if $f(g\cdot x) = g\cdot f(x)$ for all $g\in G$. Similarly, a function $f$ is called $G$-invariant if $f(g\cdot x) = f(x)$ for all $g\in G$. 

\subsection{$G$-invariant networks } 
$G$-equivariant networks are a popular way to model $G$-equivariant functions. These networks are composed of several linear $G$-equivariant layers, interleaved with activation functions like ReLU, and have the following form: 
\begin{equation}
    \label{eq:equivariant}
    f = L_k\circ \sigma \circ L_{k-1} \dots \circ \sigma \circ L_1,
\end{equation}
Where $L_i:\R^{\ell\times d_i}\too \R^{\ell\times d_{i+1}}$ are linear $G$-equivariant layers, $d_i$ are the feature dimensions and $\sigma$ is a point-wise activation function. It is straightforward to show that this architecture results in a $G$-equivariant function. $G$-invariant networks are defined by adding an invariant layer on top of a $G$-equivariant function followed by a multilayer Perceptron (MLP), and have the form:
\begin{equation}
    \label{eq:invariant}
 g = m\circ \sigma \circ h \circ \sigma \circ f \,,
\end{equation}

where $h:\R^{\ell\times d_{k+1}}\too \R^{d_{k+2}}$ is a linear $G$-invariant layer and  $m:\R^{d_{k+2}}\too \R^{d_{k+3}}$ is an MLP. It can be readily shown that this architecture results in a $G$-invariant function. 

\subsection{Characterizing equivariant layers} 
The main building block of $G$-invariant/equivariant networks are linear $G$-invariant/equivariant layers. To implement these networks, one has to characterize the space of linear $G$-invariant/equivariant layers, namely, $L_i,h$ in Equations (\ref{eq:equivariant}-\ref{eq:invariant}). For example, it is well known that for images with the group $G$ of circular 2D translations, the space of linear $G$-equivariant layers is simply the space of all 2D convolutions operators \cite{puschel2008algebraic}. Unfortunately, such elegant characterizations are not available for most permutation groups.

Characterizing linear $G$-equivariant layers can be reduced to the task of solving a set of linear equations in the following way: We are looking for a linear operator  $L:\R^\ell \too \R^\ell $ that commutes with all the elements in $G$, namely: 
\begin{equation}\label{e:equivariance}
    L(g\cdot x)=g\cdot L(x),~ x\in \R^\ell, ~g\in G.
\end{equation}
Note that $L$ can be realized as a $\ell \times \ell$ matrix (which will be denoted in the same way), and as in \citet{maron2018invariant}, Equation \ref{e:equivariance} is equivalent to the following linear system: 
\begin{equation}\label{e:fixed_point}
    g\cdot L = L, ~g\in G,
\end{equation}
where $g$ acts on both dimensions of $L$. The solution space of Equation \ref{e:fixed_point} characterizes the space of all $G$-equivariant linear layers, or equivalently, defines a parameter sharing scheme on the layer parameters for the group $G$ \cite{wood1996representation,Ravanbakhsh2017}. We will denote the dimension of this space as $E(G)$. We note that in many important cases (\eg, \cite{zaheer2017deep,Hartford2018, maron2018invariant, albooyeh2019incidence}) $|G|$ is exponential in $\ell$ so it is not possible to solve the linear system naively, and one has to resort to other strategies.

\subsection{Deep Sets} 
Since the current paper generalizes \textit{DeepSets} \cite{zaheer2017deep}, we summarize their main results for completeness. Let $\set{x_1,\dots x_n} \subset{\R}$ be a set, which we represent in arbitrary order as a vector $x\in \R^n$. DeepSets characterized all $S_n$-equivariant layers, namely, all matrices $L\in \R^{n\times n}$ such that $g\cdot L(x)=L(g\cdot x) $ for any permutation $g\in S_n$ and have shown that these operators have the following structure: $L=\lambda I_n + \beta \ones\ones^T $. When considering sets with higher dimensional features, \ie, $x_i\in \R^d$ and $X\in \R^{n\times d}$, this characterization takes the form: 
\begin{equation}\label{e:deepsets}
L(X)_i= L_1(x_i) + L_2\left(\sum_{j\neq i}^n x_j\right), 
\end{equation}
where $L_1,L_2:\R^d \too \R^{d}$ are general linear functions and the subscript represents the $i$-th row of the output. The paper then suggests to concatenate several such layers, yielding a deep equivariant model (or an invariant model if a set pooling layer is added on top). \citet{zaheer2017deep,qi2017pointnet} established the universality of invariant networks that are composed of DeepSets Layers and \citet{segol2019universal} extended this result to the equivariant case.

\section{DSS layers}\label{s:uniform}

Our main goal is to design deep models for sets of elements with non-trivial per-element symmetries. In this section, we first formulate the symmetry group $G$ of such sets. The deep models we advocate are composed of linear $G$-equivariant layers (DSS layers), therefore, our next step is to find a simple and practical characterization of the space of these layers. 



\subsection{Sets with symmetric elements}
Let $\set{x_1,\dots x_n} \subset{\R^{d}}$ be a set of elements with symmetry group $H\leq S_d$. We wish to characterize the space of linear maps $L:\R^{n\times d}\too \R^{n\times d}$ that are equivariant to both the natural symmetries of the elements, represented by the elements of the group $H$, as well as to the order of the $n$ elements, represented by $S_n$.

In our setup, $H$ operates on all elements $x_i$ in the same way. More formally, the symmetry group is defined by $G=S_n \times H$, where $S_n$ is the symmetric group on $n$ elements. This group operates on $X\in \R^{n\times d}$ by applying the permutation $q\in S_n$ to the first dimension and the same element $h\in H$ to the second dimension, namely $\left((q,h)\cdot X\right)_{ij} = X_{q^{-1}(i)h^{-1}(j)}$. Figure \ref{fig:setup} illustrates this setup. 
\begin{wrapfigure}[16]{r}{0.2\textwidth}
\centering
    \vspace{-7pt}\includegraphics[width=0.11\textwidth]{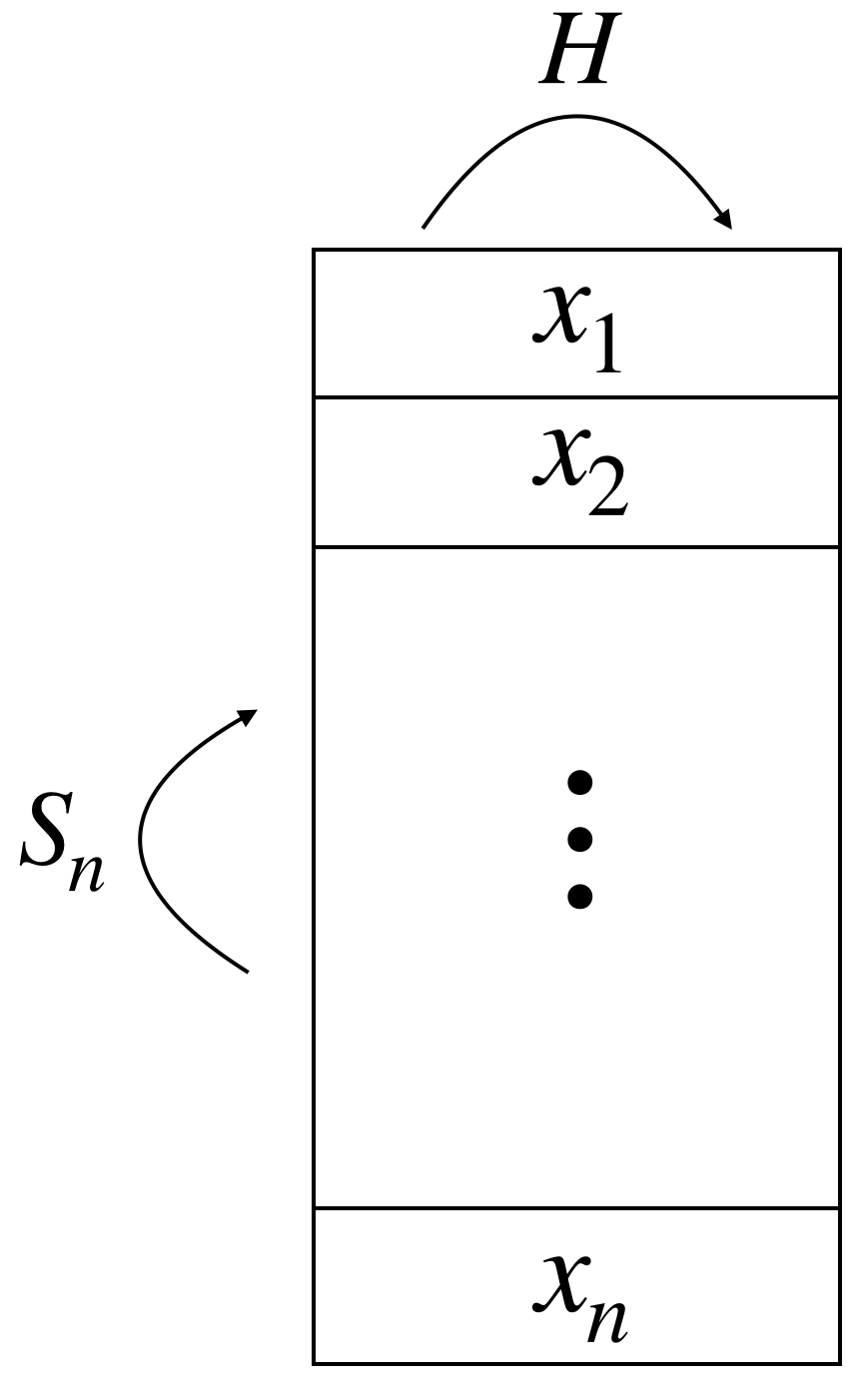}
    \caption{The input to a DSS layer is an $n\times d$ matrix, in which each row holds a $d$-dimensional element. $G=S_n\times H$ acts on it by applying a permutation to the columns and an element $h\in H$ to the rows.}
    \label{fig:setup}
\end{wrapfigure}
Notably, this setup generalizes several popular learning setups: (1) DeepSets, where $H=\{I_d\}$ is the trivial group. (2) Tabular data \cite{Hartford2018}, where $H=S_d$. (3) Sets of images, where $H$ is the group of circular translations \cite{aittala2018burst}.

One can also consider another setup, where the members of $H$ that are applied to each element of the set may differ. Section \ref{s:semidirect} of the supplementary material formulates this setup and characterizes the corresponding equivariant layers in the common case where $H$ acts transitively on $\set{1,\dots,d}$. While this setup can be used to model several interesting learning scenarios, it turns out that the corresponding equivariant networks are practically reduced to Siamese networks that were suggested in previous works.

\subsection{Characterization of equivariant layers }
This subsection provides a practical characterization of linear $G$-equivariant layers for $G=S_n\times H$. Our result generalizes DeepSets (\eqref{e:deepsets}) whose layers are tailored for $H=\set{I_d}$, by replacing the linear operators $L_1,L_2$ with linear $H$-equivariant operators. This result is summarized in the following theorem:

\begin{theorem}\label{thm:uniform}
    Any linear $G-$equivariant layer $L:\R^{n\times d}\too \R^{n\times d}$ is of the form 
    \[L(X)_{i}=L_1^H(x_i) + L_2^H\left(\sum_{j\neq i}^{n}x_j\right),\]
    where $L_1^H,L_2^H$ are linear $H$-equivariant functions 
\end{theorem}

Note that this is equivalent to the following formulations $L(X)_{i}=L_1^H(x_i) +  L_2^H(\sum_{j=1}^{n}x_j)= L_1^H(x_i) + \sum_{j=1}^{n} L_2^H(x_j)$ due to linearity, and we will use them interchangeably throughout the paper.  Figure \ref{fig:layer} illustrates Theorem \ref{thm:uniform} for sets of images. In this case, applying a DSS layer amounts to: (i) Applying the same convolutional layer $L_1$ to all images in the set (blue); (ii) Applying another convolutional layer  $L_2$  to the sum of all images (orange); and (iii) summing the outputs of these two layers. We discuss this theorem in the context of other widely-used data types such as point-clouds and graphs in section \ref{s:examples_and_generalizations} of the Supplementary material.

We begin the proof by stating a useful lemma, that provides a formula for the dimension of the space of linear $G$-equivariant maps:
\begin{lemma}\label{l:equi_dim}
Let $G \leq S_\ell$, then the dimension of the space of $G$-equivariant linear functions $L:\R^\ell\too \R^\ell$ is 
$$E(G)=\frac{1}{|G|}\sum_{g\in G}\trace(P(g))^2 ,$$
where $P(g)$ is the permutation matrix that corresponds to the permutation $g$.
\end{lemma}

The proof is given in the supplementary material. Given this lemma we can now prove Theorem \ref{thm:uniform}:

\begin{proof}[Proof of Theorem \ref{thm:uniform}]
We wish to prove that all linear $G$-equivariant layers $L:\R^{n\times k}\too \R^{n\times k}$ are of the form $L(X)_i=L_1^H(x_i) + L_2^H(\sum_{j\neq i}^{n}x_j)$. Clearly, layers of this form are linear and equivariant. Moreover, the dimension of the space of these operators is exactly $2E(H)$ since we need to account for two linearly independent $H$-equivariant operators. The linear independence follows from the fact that their support in the matrix representation of $L$ is disjoint. On the other hand, using Lemma \ref{l:equi_dim}  we have:
\begin{align*} 
E(G) &= \frac{1}{|G|}\sum_{g\in G}\trace(P(g))^2 =\\
&=\frac{1}{|H|}\frac{1}{n!}\sum_{q\in S_n}\sum_{h\in H}\trace(P(q) \otimes P(h))^2\\
&=\frac{1}{|H|}\frac{1}{n!}\sum_{q\in S_n}\sum_{h\in H}\trace(P(q))^2 \trace(P(h))^2\\
&=\left(\frac{1}{|H|}\sum_{h\in H}\trace(P(h))^2\right)\cdot \left(\frac{1}{n!}\sum_{q\in S_n}\trace(P(q))^2\right)\\
&=E(H) E(S_n)=2E(H).
\end{align*}
Here we used the fact that the trace is multiplicative with respect to the Kronecker product as well as the fact that $E(S_n)=2$ (see \citep{zaheer2017deep} or Appendix 2 in \citep{maron2018invariant} for a generalization of this result).

To conclude, we have a linear subspace $\set{L \mid L(X)_{i}=L_1^H(x_i) + L_2^H(\sum_{j\neq i}^{n}x_j)}$, which is a subspace of the space of all linear $G$-equivariant operators, but has the same dimension, which implies that both spaces are equal.
\end{proof}

\paragraph{Relation to \cite{aittala2018burst,sridhar2019multiview}.} In the specific case of a set of images and translation equivariance, $L_i^H$ are convolutions. In this setting, \cite{aittala2018burst,sridhar2019multiview} have previously proposed using set-aggregation layers after convolutional blocks. The main differences between these studies and the current paper  are: (1) Our work applies to all types of symmetric elements and not just images; (2) We derive these layers from first principles; (3) We provide a theoretical analysis (Section \ref{s:universal}); (4) We apply an aggregation step at each layer instead of only after convolutional blocks.


\paragraph{Generalizations.} Section \ref{s:AppendixCharacterization} of the supplementary material generalizes Theorem \ref{thm:uniform} to equivariant linear layers with multiple features. It also generalizes to several  additional types of equivariant layers: $L:\R^{n\times d}\too \R$, $L:\R^{n\times d}\too \R^n$ and $L:\R^{n\times d}\too \R^d$. 

\paragraph{Product of arbitrary permutation groups.} See Section \ref{s:group_prod} of the supplementary material for further discussion and characterization of the space of equivariant maps for a product of arbitrary permutation groups. 

\section{A universal approximation theorem}\label{s:universal}
When restricting a network to be invariant (equivariant) to some group action, one may worry that these restrictions could reduce the network expressive power (see \citet{maron2019universality} or \citet{xu2018how} for concrete examples). We now show that networks that are constructed from DSS layers do not suffer from loss of expressivity. Specifically, we show that for any group $H$ that induces a \textit{universal} $H$-invariant (equivariant) network, its corresponding  $G$-invariant (equivariant) network \revision{has high expressive power: we prove universal approximation for invariant and equivariant functions defined on any compact set with zero intersection with a specific low-dimensional set $\mathcal{E}\subset \R^{n\times d}$. The precise definition of $\mathcal{E}$ is given in the supplementary. }

We first state a lemma, which we later use for proving our universal-approximation theorems. The lemma shows that one can uniquely encode orbits of a group $H$ in an invariant way by using a polynomial function. The full proof is given in Section \ref{s:AppendixProofs} of the supplementary material.
\begin{lemma}\label{l:encoding}
Let $H\leq S_d$ then there exists a polynomial function $u:\R^d\too \R^l$, for some $l\in \mathbb{N}$, for which $u(x)=u(y)$ if and only if $x=h\cdot y$ for some $h\in H$. 
\end{lemma}
\begin{proof}[Proof idea]
This lemma is a generalization of Proposition 1 in \cite{maron2019provably} and we follow their proof. The main idea is that for any such group $H$ there exists a finite set of invariant polynomials whose values on $\R^d$ uniquely define each orbit of $H$ in $\R^d$. 
\end{proof}

\subsection{Invariant functions}
We are now ready to state and prove our first universal approximation theorem. As before, the full proof can be found in the supplementary material (Section \ref{s:AppendixProofs}). 
\begin{theorem}\label{thm:invariant_universality}
Let $K \subset \R^{n\times d}$ be a compact domain such that $ K=\cup_{g\in G}gK$ \revision{and $K\cap \mathcal{E}=\emptyset$}. $G$-invariant networks are universal approximators (in $\|\cdot \|_\infty$ sense) of continuous $G$-invariant functions on $K$ if $H$-invariant networks are universal\footnote{We assume that there is a universal approximation theorem for general continuous functions and fully connected networks   that use our activation functions, \eg, ReLU.}. 
\end{theorem}
\begin{proof}[Proof idea]
Let $f:K\too \R$ be a continuous $G$-invariant function we wish to approximate. The idea of the proof is as follows: \revision{(1) We compute an $S_n$-invariant set descriptor $\sum_{j=1}^n x_j$ and concatenate it to all elements as an extra feature channel. Later on, we compute an $H$-invariant representation for each set elements separately, thus lose information regarding their relative position. This shared descriptor will allow us to synchronize these individual invariant representations into a global $G$-invariant representation. The excluded set $\mathcal{E}$ is the set of points where this descriptor is degenerate and does not allow us to correctly synchronize the separate elements.}  (2) we encode each element $x_i$ \revision{and the extra descriptor $\sum_{j=1}^n x_j$} with a unique $H$-invariant polynomial descriptor $u_H(x_i,\sum_{j=1}^n x_j)\in \R^{l_H}$ (3) we encode the resulting set of descriptors with a unique $S_n$-invariant polynomial set descriptor\revision{ $u_{S_n}\left(\{u_{H}(x_i,\sum_{j=1}^n x_j) \}_{i\in [n]}\right)\in \R^{l_{S_n}}$ (4) we map the unique set descriptor $u_{S_n}\left(\{u_H(x_i,\sum_{j=1}^n x_j) \}_{i\in [n]}\right)$} to the appropriate value defined by $f$ (5) we use the classic universal approximation theorem \cite{cybenko1989approximation, hornik1989multilayer} and our assumption on the universality of $H$-invariant networks to conclude that there exists a $G$-invariant network that can approximate each one of the previous stages to arbitrary precision on $K$.
\end{proof}

\paragraph{Relation to Siamese networks.} \revision{If we omit step (1) in the proof of Theorem \ref{thm:invariant_universality}, i.e., calculating the sum of all the inputs and concatenating it to each element, the proof implies that simple Siamese architectures that apply an $H$-invariant network to each element in the set followed by a sum aggregation and finally an MLP,  are universal with respect to the wreath-product symmetries (see discussion at section \ref{s:semidirect} of the supplementary material), which is a strictly larger group than $G$ (assuming $H$ is not trivial). Since $G$-invariant networks can simulate any such Siamese network, Theorem \ref{thm:invariant_universality} implies that $G$-invariant networks are strictly more expressive then Siamese networks. This is because $G$-invariant networks can distinguish different $G$-orbits that fall into the same wreath-product orbits }. In section \ref{s:exp}, we compare the Siamese architecture to our DSS networks and show that DSS-based architectures perform better in practice on several tasks.

\paragraph{Relation to \cite{maron2019universality}.} The authors proved that for any permutation group $G$, $G$-invariant networks have a universal approximation property, if the networks are allowed to use high-order tensors as intermediate representations (\ie, $X\in \R^{d^l}$ for $2 \leq l\leq n^2$), which are computationally prohibitive. We \revision{provide a complementary} result by proving that if first-order \footnote{First-order networks use only first-order tensors.} $H$-invariant networks are universal, so are first-order $G$-invariant networks. 

\subsection{Equivariant functions}
Three possible types of equivariant functions can be considered. First, functions of the form $f:\R^{n\times d}\too \R^n$. For example, such a function can model a selection task in which we are given a set $\{x_1,\dots,x_n\}$ and we wish to select a specific element from that set. Second, functions of the form $f:\R^{n\times d}\too \R^d$. An example for this type of functions would be an image-deblurring task in which we are given several noisy measurements of the same scene and we wish  to generate a single high quality image (\eg, \cite{aittala2018burst}). Finally,  functions of the form $f:\R^{n\times d}\too \R^{n\times d}$. This type of functions can be used to model tasks such as image co-segmentation where the input consists of several images and the task is to predict a joint segmentation map.

In this subsection we will prove a similar universality result for the third type of $G$-equivariant functions that were mentioned above, namely $f:\R^{n\times d}\too \R^{n\times d}$. We note that the equivariance of the first and second types can be easily deduced from this case. One can transform, for example, an  $\mathbb{R}^{n\times d}\too \R^{d}$ $G$-equivariant function into a $\R^{n\times d}\too \R^{n\times d}$ function by repeating the $\R^d$ vector $n$ times and use our general approximation theorem on this function. We can get back a  $\mathbb{R}^{n\times d}\too \R^{d}$ function by averaging over the first dimension.

\begin{theorem}\label{thm:equivariant_universality}
Let $K \subset \R^{n\times d}$ be a compact domain such that $ K=\cup_{g\in G}gK$ \revision{and $K\cap \mathcal{E}=\emptyset$}. $G$-equivariant networks are universal approximators (in $\|\cdot \|_\infty$ sense) of continuous $\R^{n\times d}\too \R^{n\times d}$ $G$-equivariant functions on $K$ if $H$-equivariant networks are universal.
\end{theorem}

\begin{proof} [Proof idea] The proof follows a similar line to the universality proof in \cite{segol2019universal}: First, we use the fact that equivariant polynomials are dense in the space of continuous equivariant functions. This enables us to assume that the function we wish to approximate is a $G$-equivariant polynomial. Next we show that for every output element, the mapping $\R^{n\times d}\too \R^d$ can be written as a sum of $H$-equivariant base polynomials with invariant coefficients. The base polynomials can be approximated by our assumption on $H$ and the invariant mappings can be approximated by leveraging a slight modification of theorem \ref{thm:invariant_universality}. Finally we show how we can combine all the parts and approximate the full function with a $G$-equivariant network. 
\end{proof}

The full proof is given in Section \ref{s:AppendixProofs} of the supplementary material. 

\subsection{Examples}
We can use Theorems  (\ref{thm:invariant_universality}-\ref{thm:equivariant_universality}) to show that DSS-based networks are universal in two important cases. For tabular data, which was considered by \citet{Hartford2018}, the symmetries are $G=S_n\times S_d$. From the universality of $S_n$-invariant and equivariant networks \cite{zaheer2017deep, segol2019universal} we get that $G$-invariant (equivariant) networks are universal as well\footnote{ \citet{Hartford2018} also considered interactions between more than two sets with $G=S_n\times S_{d_1}\times \dots \times S_{d_k}$. Our theorems can be extended to that case by induction on $k$.}. For sets of images, when $H$ is the group of circular translations, it was shown in \citet{yarotsky2018universal} that $H$-invariant/equivariant networks are universal\footnote{We note that this paper considers convolutional layers with full size kernels and no pooling layers}, which implies universality of our DSS models. 



\section{Experiments}\label{s:exp}
In this section we investigate the effectiveness of DSS layers in practice, by comparing them to previously suggested architectures and different aggregation schemes. We use the experiments to answer two basic questions: \textbf{(1)~\textit{Early or late aggregation?}} Can \emph{early aggregation} architectures like DSS and its variants improve learning compared to \emph{Late aggregation} architectures, which fuse the set information at the end of the data processing pipeline? and \textbf{(2)~\textit{How to aggregate?}} What is the preferred early aggregation scheme?

\paragraph{Tasks.} We evaluated DSS in a series of six experiments spanning a wide range of tasks: from classification ($\R^{n\times d}\too \R$), through selection ($\R^{n\times d}\too \R^{n}$) and burst image deblurring  ($\R^{n\times d}\too \R^{d}$) to general equivariant tasks ($\R^{n\times d}\too \R^{n\times d}$). The experiments also demonstrate the applicability of DSS to a range of data types, including point-clouds, images and graphs. Figure \ref{fig:exp} illustrates the various types of tasks evaluated. A detailed description of all tasks, architectures and datasets is given in the supplementary material (Section \ref{s:implementation_details}).

\paragraph{Competing methods.}
We compare DSS to four  other models: (1) MLP; (2) DeepSets (DS) \cite{zaheer2017deep}; (3) Siamese network; (4) Siamese network followed by DeepSets (Siamese+DS). 

We also compare several variants of our DSS layers:  \textbf{(1)~DSS(sum)}: our basic DSS layer from Theorem \ref{thm:uniform} \textbf{(2)~DSS(max)}: DSS with max-aggregation instead of sum-aggregation  \textbf{(3)~DSS(Aittala)}: DSS with the aggregation proposed in  \cite{aittala2018burst}, namely, $L(x)_i \mapsto [L^H(x_i),\max_{j=1}^n L^H(x_j)]$ where $[]$ denotes feature concatenation and $L^H$ is a linear  $H$-equivariant layer \textbf{(4)~DSS(Sridhar)}: DSS layers with the aggregation proposed in  \cite{sridhar2019multiview} ,\ie,  $L(x)_i\mapsto L^H(x_i)- \frac{1}{n} \sum_{j=1}^n L^H(x_j)$.

\paragraph{Evaluation protocol.} For a fair comparison, for each particular task, all models have roughly the same number of parameters. In all experiments, we report the mean and standard deviation over 5 random initializations. Experiments were conducted using NVIDIA DGX with V100 GPUs.
\begin{figure}[t] 
    \centering
    \includegraphics[width=0.9\columnwidth]{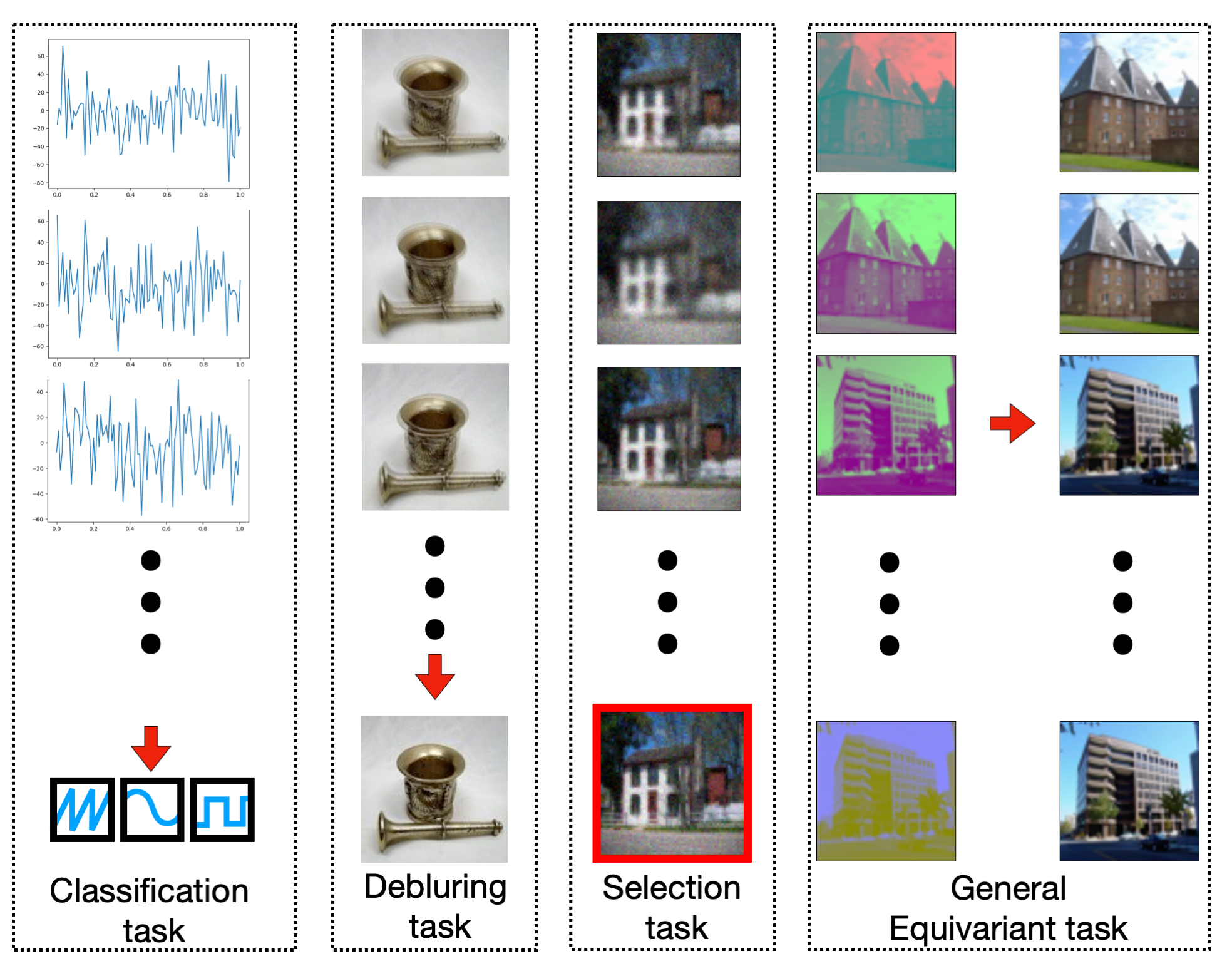}
    \caption{We consider all possible types of invariant and equivariant learning tasks in our settings: classification ($\R^{n\times d}\too \R$), selection ($\R^{n\times d}\too \R^{n}$), merging ($\R^{n\times d}\too \R^{d}$) and general equivariant tasks ($\R^{n\times d}\too \R^{n\times d}$). }
    \label{fig:exp}
\end{figure}

\subsection{Classification with multiple measurements}

\begin{table*}[th]
  \centering
  \scriptsize
    \begin{tabular}{|l|l|l|llll|c|}
        \hline
         \multirow{2}{*}{Dataset} & \multirow{2}{*}{Data type} & \multicolumn{1}{|c|}{Late Aggregation} & \multicolumn{4}{|c|}{Early Aggregation}& \multirow{2}{*}{Random choice}\\
                  & & Siamese+DS & DSS (sum)& DSS (max) &DSS (Sridhar) & DSS (Aittala)& \\
        \hline
        UCF101 & Images & 36.41\% $\pm$ 1.43 &  76.6\%  $\pm$ 1.51&76.39\% $\pm$ 1.01 &60.15\% $\pm$ 0.76&\textbf{77.96\%} $\pm$ 1.69&12.5\% \\ 
        Dynamic Faust  & Point-clouds & 22.26\% $\pm$ 0.64 &  {42.45\%}  $\pm$ 1.32& 28.71\% $\pm$ 0.64 & \textbf{54.26}\% $\pm$ 1.66 &26.43\% $\pm$ 3.92 &14.28\%\\ 
        Dynamic Faust  & Graphs & 26.53\% $\pm$ 1.99 &  44.24\%  $\pm$ 1.28&30.54\% $\pm$ 1.27 &\textbf{53.16}\% $\pm$ 1.47 &26.66\% $\pm$ 4.25 &14.28\%\\ 
        \hline
    \end{tabular}%
    \caption{Frame  selection  tasks for  images, point-clouds and graphs. Numbers represent average classification accuracy.}        \label{t:first_frame}
\end{table*}

To illustrate the benefits of DSS, we first evaluate it in a signal-classification task using a synthetic dataset that we generated. Each sample consists of a set of $n=25$ noisy measurements of the same 1D periodic signal sampled at 100 time-steps (see Figure \ref{fig:exp}). The clean signals are sampled uniformly from three signal types - sine, saw-tooth and square waves - with varying amplitude, DC component, phase-shift and frequency. The task is to predict the signal type given the set of noisy measurements. Figure \ref{fig:signal} depicts the classification accuracy as a function of varying training set sizes, showing that DSS(sum) outperforms all other methods. Notably, DSS(sum) layers achieve significantly higher accuracy then the DeepSets architecture which takes into account the set structure but not within-element symmetry. DSS(sum) also outperforms the the \textit{Siamese} and \textit{Siamese+DS} architectures, which do not employ early aggregation. \textit{DSS(Sridhar)} fails, presumably  because it employs a mean removal aggregation scheme which is not appropriate for this task (removes the signal and leaves the noise). 


\begin{figure}[t] 
\centering
    \includegraphics[width=\columnwidth]{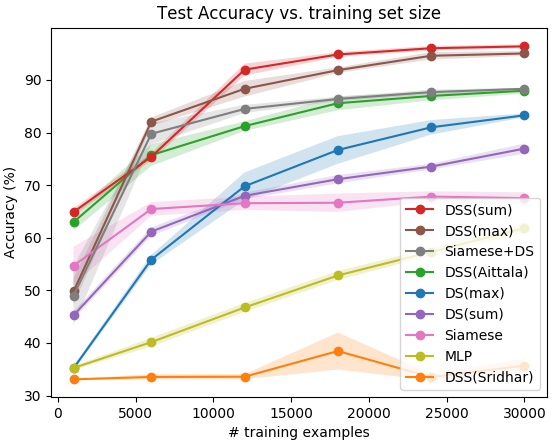}
    \caption{Comparison of set learning methods on the signal classification task. Shaded area represents standard deviation.}
    \label{fig:signal}
\end{figure}

\begin{table*}[t]
  \scriptsize
  \centering
    \begin{tabular}{|l|l|llll|c|}
  \hline
   \multirow{2}{*}{Noise type and strength}  & \multicolumn{1}{|c|}{Late Aggregation} & \multicolumn{4}{|c|}{Early Aggregation}& \multirow{2}{*}{Random choice}\\
          & Siamese+DS    & DSS (sum) & DSS (max) & DSS (Sridahr) & DSS (Aittala) &\\
  \hline
    Gaussian $\sigma=10$ & 77.2\% $\pm$ 0.37 & \textbf{78.48\%} $\pm$ 0.48&77.99\% $\pm$ 1.1 &76.8\% $\pm$ 0.25 & 78.34\% $\pm$ 0.49 &5\% \\
    Gaussian $\sigma=30$ & 65.89\% $\pm$ 0.66 & \textbf{68.35\%} $\pm$ 0.55&67.85\% $\pm$ 0.40 &61.52\% $\pm$ 0.54 & 66.89\% $\pm$ 0.58 &5\% \\
    Gaussian $\sigma=50$ & 59.24\% $\pm$ 0.51 & \textbf{62.6\%} $\pm$ 0.45&61.59\% $\pm$ 1.00 &55.25\% $\pm$ 0.40 &62.02\% $\pm$ 1.03 &5\% \\
      \hline

    Occlusion $10\%$ & 82.15\% $\pm$ 0.45 & 83.13\%$\pm$  1.00& \textbf{83.27} $\pm$ 0.51 &83.21\% $\pm$ 0.338 & 83.19\% $\pm$  0.67 &5\% \\
    Occlusion $30\%$ & 77.47\% $\pm$ 0.37 & 78\% $\pm$ 0.89&78.69\% $\pm$ 0.32 &\textbf{78.71\%} $\pm$ 0.26 &78.27\% $\pm$ 0.67 &5\% \\
    Occlusion $50\%$ & 76.2\% $\pm$ 0.82 & \textbf{77.29\%} $\pm$ 0.40& 76.64\% $\pm$ 0.45 &77.04\% $\pm$ 0.75 &77.03\% $\pm$ 0.58 &5\%  \\
  \hline
    \end{tabular}%
\caption{Highest-quality image selection. Values indicate the mean accuracy.}
  \label{t:quality_and _ordering}%

\end{table*}%

\begin{table*}[h]
\centering
  \scriptsize
    \begin{tabular}{|l|l|llll|c|}
    \hline 
   \multirow{2}{*}{Task}  & \multicolumn{1}{|c|}{Late Aggregation} & \multicolumn{4}{|c|}{Early Aggregation}& \multirow{2}{*}{TP}\\
          & Siamese+DS    & DSS (sum) & DSS (max) & DSS (Sridahr) & DSS (Aittala) &\\
    \hline
    Color matching (places) & 8.06 $\pm$ 0.06 & 1.78 $\pm$ 0.03&  1.92 $\pm$ 0.07    &1.97 $\pm$ 0.02&\textbf{1.67} $\pm$ 0.06& 14.68 \\
    Color matching (CelebA) &  6 $\pm$ 0.13 & 1.27 $\pm$ 0.07&1.34 $\pm$ 0.07 &1.35 $\pm$ 0.03&\textbf{1.17} $\pm$ 0.04& 18.72 \\
    \hline
    Burst deblurring (Imagenet) &  6.15 $\pm$ 0.05 & 6.11 $\pm$ 0.08& 5.87 $\pm$ 0.05 &21.01 $\pm$ 0.08&\textbf{5.7} $\pm$ 0.13& 16.75 \\
    \hline
    \end{tabular}%
    \caption{Color-channel matching and burst deblurring tasks. Values indicate mean absolute error per pixel over the test set where the pixel values are in $[0,255]$. TP stands for the trivial grey-scale predictor.}
  \label{t:color}%
\end{table*}

\subsection{Selection tasks}
We next test DSS layers on selection tasks. In these tasks, we are given a set and wish to choose one element of the set that obeys a predefined property. Formally, each task is modelled as a $G$-equivariant function $f:\R^{n\times d}\too \R^n$, where the output vector represents the probability of selecting each element. The architecture comprises of three convolutional blocks employing Siamese or DSS variants, followed by a DeepSets block. We note that the \textit{Siamese+DS} model was suggested for similar selection tasks in \cite{zaheer2017deep}. 
\vspace{-15pt}
\paragraph{Frame selection in images and shapes.} The first selection task is to find a particular frame within an unordered set of frames extracted from a video/shape sequence. For videos, we used the UCF101 dataset \cite{soomro2012ucf101}. Each set contains $n=8$ frames that were generated by randomly drawing a video, a starting position and frame ordering. The task is to select the "first" frame, namely, the one that appeared earliest in the video.  Table \ref{t:first_frame} details the accuracy of all compared methods in this task, showing that \textit{DSS(sum)} and \textit{DSS(Aittala)} outperform \textit{Siamese+DS} and \textit{DSS(Sridhar)} by a large margin.

In a second selection task, we demonstrate that DSS can handle multiple data types. Specifically, we showcase how DSS operates on point-clouds and graphs. Given a short sequences of 3D human shapes preforming various activities, the task is to identify which frame was the center frame in the original non-shuffled sequence.  These human shapes are represented as point-clouds in the first experiment and as graphs (point-clouds + connectivity) in the second. 

\begin{figure}[ht]
\centering
    \vspace*{-0cm}
    \includegraphics[width=0.98\linewidth,trim=0 10 0 20,clip]{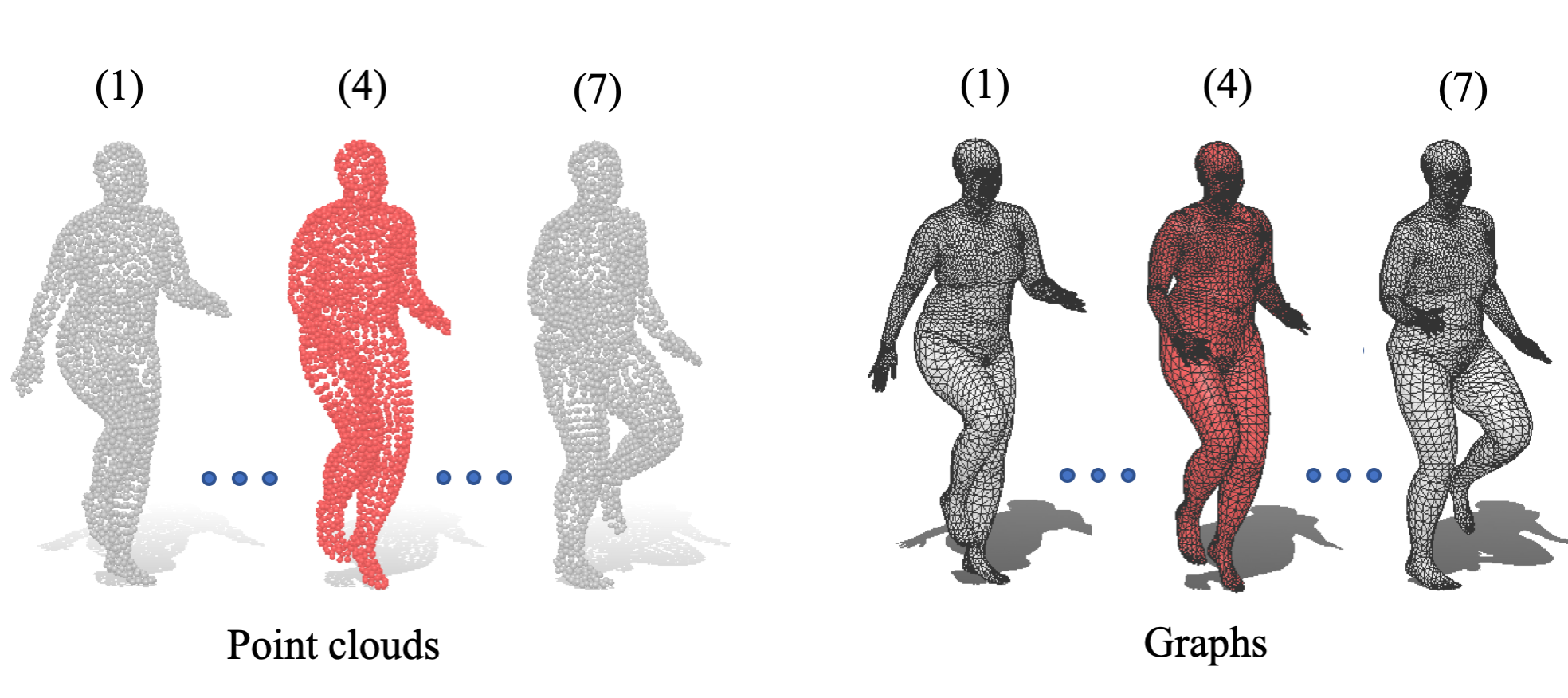}
    \caption{Shape-selection task on human shape sequences. Shapes  are represented as graphs or as point-clouds. The task is to select the central frame (red). Numbers indicate frame order.}
    \label{fig:pointclouds_and_graph}
\end{figure}

To generate the data, we cropped 7-frame-long sequences from the Dynamic Faust dataset \cite{dfaust:CVPR:2017} in which the shapes are given as triangular meshes. To generate point-clouds, we simply use the mesh vertices. To generate graphs, we use the graph defined by the triangular mesh \footnote{In \cite{dfaust:CVPR:2017} the points of each mesh are ordered consistently, providing  point-to-point correspondence across frames. When this correspondence is not available, a shape matching algorithm like \cite{litany2017deep,maron2018probably} can be used as  preprocessing.}. See Figure \ref{fig:pointclouds_and_graph} for an illustration of this task. 

Results are summarized in Table \ref{t:first_frame}, comparing DSS variants to a late-aggregation baseline (Siamese +DS) and to random choice. 
We further compared to a simple yet strong baseline. Using the mapping between points across shapes, we computed the mean of each point, and searched for the shape that was closest to that mean in $L_1$ sense. Frames in the sequence are $~80 msec$ apart, which limits the deviations around the mean, making it a strong baseline. Indeed, it achieved an accuracy of $34.47$, which outperforms both late aggregation, DSS(max) and DSS(Aitalla). In contrast, sum-based early aggregation methods reach significantly higher accuracy. 
Interestingly, using a graph representation provided a small improvement over point-clouds for almost all methods . 
\vspace{-10pt}
\paragraph{Highest quality image selection.} Given a set of $n=20$ degraded images of the same scene, the task is to select the highest-quality image. We generate data for this task from the Places dataset \cite{zhou2017places}, by adding noise and Gaussian blur to each image. The target image is defined to be the image that is the most similar in $L_1$ norm sense to the original image (see Figure \ref{fig:exp} for an illustration). Notably, DSS consistently improves over Siamese+DS  with a margin of $1\%$ to $3\%$. See Table \ref{t:quality_and _ordering}. 

\subsection{Color-channel matching}
To illustrate the limitation of late-aggregation, we designed a very simple image-to-image task that highlights why early aggregation can be critical: learning to combine color channels into full images. Here, each sample  consists of six images, generated from two randomly selected color images,
by separating each image into three color channels. In each mono-chromatic image two channels were set to zero, yielding a $d=64\times 64\times 3$ image. The task is to predict the fully colored image (\ie, imputing the missing color channels) for each of the set element. This can be formulated as a $\R^{n\times d}\too \R^{n\times d}$ $G$-equivariant task. See Figure \ref{fig:exp} for an example. 

We use a U-net architecture \cite{ronneberger2015u},  where convolutions and deconvolutions are replaced with Siamese layers or DSS variants. A DeepSets block is placed between the encoder and the decoder. Table \ref{t:color} shows that layers with early aggregation significantly outperform DS+Siamese. For context, we add the error value of a trivial predictor which imputes the zeroed color channels by replicating the input color channel, resulting in a gray-scale image. This experiment was conducted on two datasets: \textit{CelebA} \cite{liu2018large}, and \textit{Places} \cite{zhou2017places}. 

\subsection{Burst image deblurring}
Finally, we test DSS layers in a task of deblurring image bursts as in \cite{aittala2018burst}. In this task, we are given a set of $n=5$ blurred and noisy images of the same scene and aim to generate a single high quality image. This can be formulated as a $\R^{n\times d}\too \R^{d}$~ $G$-equivariant task. See results in Table \ref{t:color}, where we also added the mean absolute error of a trivial predictor that outputs the median pixel of the images in the burst at each pixel. More details can be found in the supplementary material.

\subsection{Summary of experiments} 

The above experiments demonstrate that applying early aggregation using DSS layers improves learning in various tasks and data types, compared with earlier architectures like \textit{Siamese+DS}.\revision{ This improvement might be attributed to the provably higher expressive power of DSS networks}.  More specifically, the basic DSS layer, \textit{DSS(sum)}, performs well on all tasks, and \textit{DSS(Aittala)} has also yielded strong results. \textit{DSS(Sridhar)} performs well on some tasks but fails on others. See Section \ref{s:multiview} of the supplementary materials for additional experiments on a multi-view reconstruction task.

\section{Conclusion}\label{s:conc}
In this paper, we have presented a principled approach for designing deep networks for sets of elements with symmetries: We have characterized the space of equivariant maps for such sets, analyzed its expressive power, exemplified its benefits over standard set learning approaches over a variety of tasks and data types and have shown that our approach generalizes several successful previous works.

\section*{Acknowledgments}
This research was supported by an Israel science foundation grant 737/18. We thank Srinath Sridhar and Davis Rempe for useful discussions.
\ignore{
}

\bibliography{OnLearningSetsOfSymmetricElements}
\bibliographystyle{icml2020}

\clearpage
\appendix
\onecolumn
\icmltitle{Supplementary material}

\section{Generalizations of equivariant layer characterization}\label{s:AppendixCharacterization}

\subsection{Equivariant layers for multiple features}
The following generalization to sets of elements with multiple features can be proved in a similar way to the section 3.1 in \cite{maron2018invariant}.
\begin{theorem}\label{thm:uniform_with_features}
Any linear $G-$equivariant layer $L:\R^{n\times d\times f}\too \R^{n\times d \times f'}$ is of the form 
\[L(X)_{i}=L_1(x_i) + L_2(\sum_{j\neq i}^{n}x_j),\]
where $L_i,~ i=1,2$ are linear $H$-equivariant functions. The dimension of the space of these layers is $2E(H)ff'$.
\end{theorem}
\subsection{General equivariant and invariant layers} \label{s:equive_layers_diff_ten_order}
In the main text we characterized all $G$-invariant functions of the form $L:\R^{n\times d}\too \R^{n\times d}$. Here, we characterize all other possibilities of equivariant and invariant functions. The proof is identical to the proof of Theorem \ref{thm:uniform} in the main paper.

\begin{theorem}
\begin{enumerate}
    \item Any linear $G-$equivariant layer $L:\R^{n\times d}\too \R^{n}$ is of the form 
$L(X)_{i}=L_1^H(x_i) + L_2^H(\sum_{j\neq i}^{n}x_j)$, where $L_i^H,~ i=1,2$ are linear $H$-invariant functions.
\item Any linear $G-$equivariant layer $L:\R^{n\times d}\too \R^{d}$ is of the form 
$L(X)=L^H(\sum_{j=1}^{n}x_j)$, where $L^H$ is linear $H$-equivariant function.
\item Any linear $G$-invariant layer $L:\R^{n\times d}\too \R$ is of the form 
$L(X)=L^H(\sum_{j=1}^{n}x_j)$, where $L^H$ is linear $H$-invariant function.

\end{enumerate}
\end{theorem}

\section{Products of arbitrary permutation groups}\label{s:group_prod}
Here, we show that Theorem \ref{thm:uniform} can be generalized to products of arbitrary permutation groups. Our first step is noting that the second part of the proof of Theorem \ref{thm:uniform} can be easily modified to show that $E(H_1\times H_2)=E(H_1)\cdot E(H_2)$ for any permutation groups $H_1,H_2$. 

Indeed, Theorem \ref{thm:uniform} is a special case of the following theorem, which characterizes the space of linear equivariant maps for arbitrary products of permutation groups.
\begin{theorem}\label{thm:arbitrary_product}
Let $H_1\leq S_n$, $H_2\leq S_d$ and $\{L_i^j\}_{i=1}^{E(H_j)},~j=1,2$ are bases for the spaces of linear $H_j$-equivariant maps. Let $G=H_1\times H_2$ act on $\R^{n\times d}$ by multiplication, $(h_1,h_2)\cdot X:=h_1Xh_2^T$. Then, a basis for the space of linear $G-$equivariant layers $L:\R^{n\times d}\too  \R^{n\times d}$ is given by  
$$ T_{i_1,i_2} = L_{i_1}^1 \otimes L_{i_2}^2,\quad  i_1=1,\dots ,E(H1),~i_2=1,\dots,E(H_2)$$

\end{theorem}

\begin{proof}
 $\set{T_{i_1,i_2}}$ is $G$-equivariant and linearly independent as a tensor product of linearly independent sets. Moreover, its size is exactly $E(H_1)\cdot E(H_2)$ so it must span the whole space of $G$-equivariant layers.
\end{proof}

The basis mentioned in Theorem \ref{thm:arbitrary_product} can be implemented using the Kronecker product identity: 

\begin{equation}\label{e:kron}
     T_{i_1,i_2}(X) = L_{i_1}^1 X L_{i_2}^{2^T}
\end{equation}
In other words, these operators can be implemented by applying $L_{i_1}^1$ to the columns of $X$ and $L_{i_2}^2$ to the rows of $X$.

To verify that Theorem \ref{thm:arbitrary_product} is indeed a generalization of Theorem \ref{thm:uniform}, consider $H_1=S_n$. A basis for $S_n$-invariant layers is given by $L^1_1 = I_n,~L^1_2 = \ones_n \ones_n^T $. From Theorem \ref{thm:arbitrary_product} it follows that a basis for the space of $G$-invariant linear layers is given by $I_n\otimes L^2_i$ and $\ones_n\ones_n^T \otimes L^2_i$ which, by Equation \ref{e:kron}, gives the basis from Theorem \ref{thm:uniform}.

\section{Equivariant layers for order dependent action}\label{s:semidirect} 

\begin{figure}[t]
    \centering
    \begin{tabular}{cc}
        \includegraphics[width=0.23\linewidth,height=0.235\linewidth]{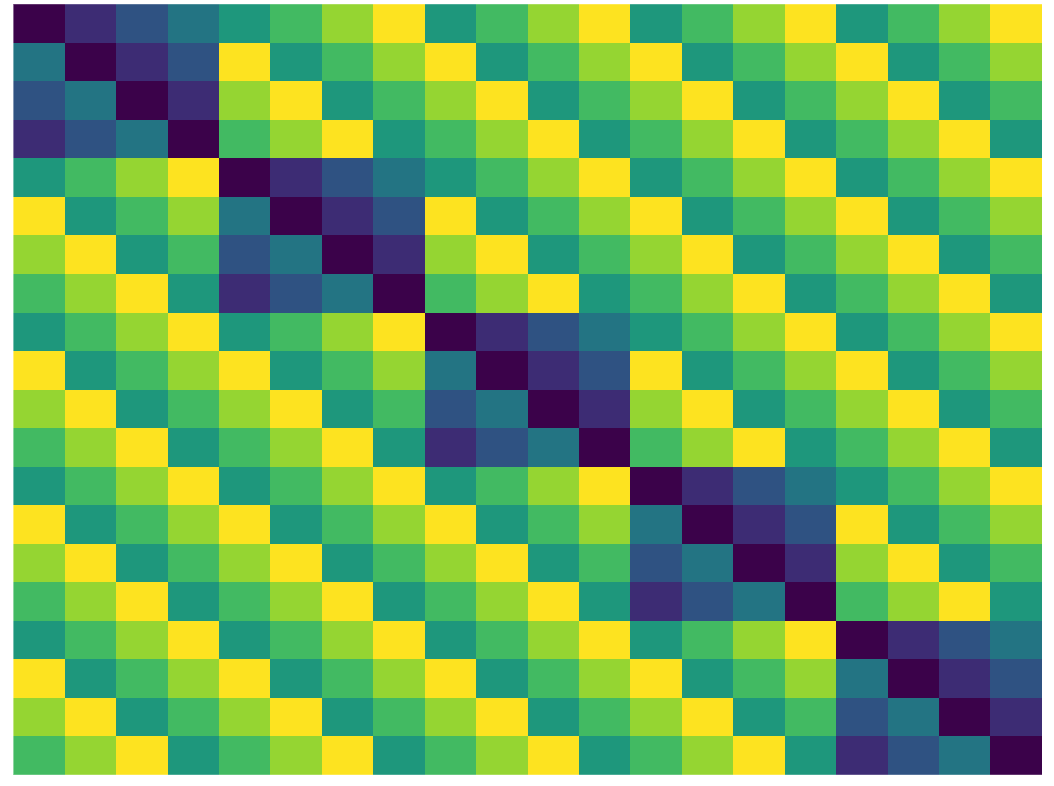} & \includegraphics[width=0.23\linewidth,height=0.235\linewidth]{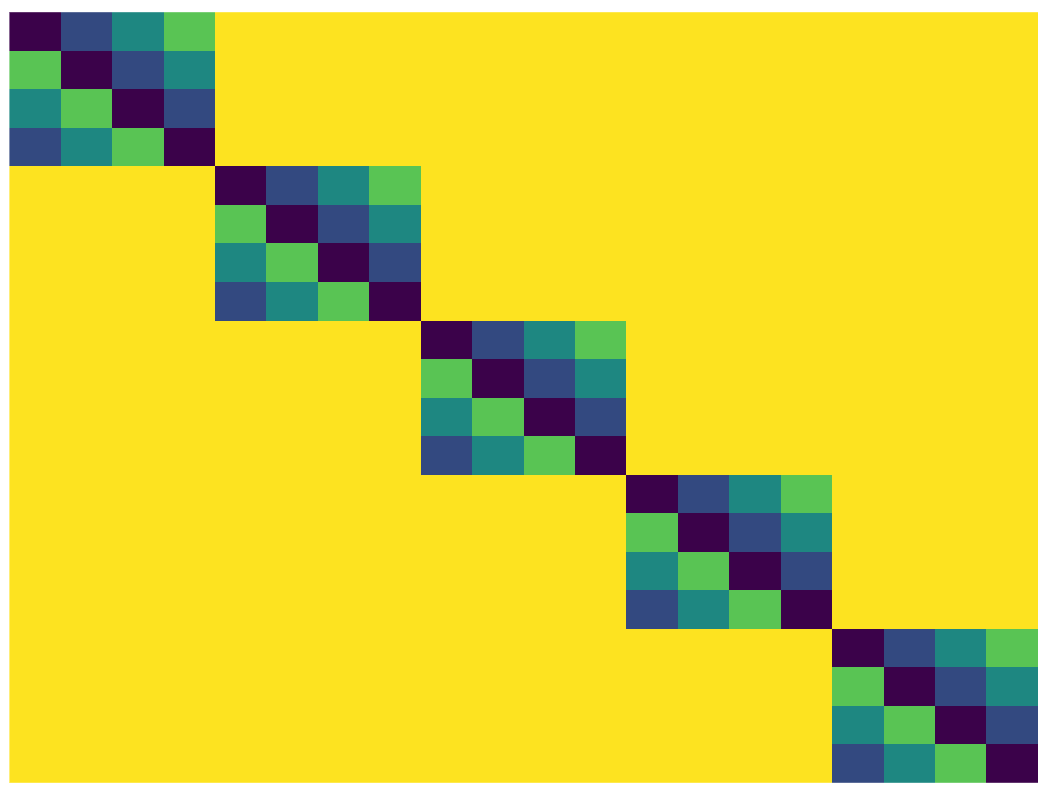}\\
        (a) & (b) 
    \end{tabular}
    \caption{parameter sharing schemes for (a) $G=S_n\times H$ and (b) $G=\oplus_{i=1}^n H\rtimes S_n$, where $d=4,n=5$ and $H=C_4$ the cyclic group of four elements. Each color represents a parameter.}
    \label{fig:param_share}
\end{figure}
As mentioned in the main text, we can consider a different learning setup, where tasks are equivariant to applying different elements of $H$ to different elements in the set. In this section, we formulate this setup and prove that when $H$ acts transitively on $\set{1,\dots,d}$, for example, in the case of images and sets, the corresponding equivariant layers are Siamese $H$-equivariant layers with an additional global summation term. In this setup, $G=\set{(h_1,\dots,h_n,\sigma)}$ is the semi-direct product $\oplus_{i=1}^n H\rtimes S_n$ (also called restricted wreath product) and the action of $G$ on $\R^{n\times d}$ is defined as  $\left((h_1,\dots,h_n,\sigma)\cdot X\right)_{ij}=X_{\sigma^{-1}(i),h_i^{-1}(j)}$.

We can now characterize the set of $G$-equivariant layers for this setup.
\begin{theorem}
If $H$ acts transitively on $\set{1,\dots,d}$ then any linear $G$-equivariant layer $L$ is of the form:
\[L(X)_{i}=L_1(x_i) + \beta \left(\sum_{j=1}^n\sum_{k=1}^d x_{jk}\right).\]
 $L_1$ is an $H$-equivariant layer and $\beta\in \R$. The dimension of the space of linear $G$-equivariant maps is $E(H)+1$.
\end{theorem}
\begin{proof}
We want to characterize the space of $G$-equivariant maps. According to \eqref{e:fixed_point}, we need to find the null space of the following fixed point equation $g\cdot L=L,~g\in G$. As shown in \cite{wood1996representation, Ravanbakhsh2017}, this is equivalent to revealing the parameter-sharing scheme that is induced by $G$, which we will define next. Let $L\in\R^{nd \times nd}$ represent a linear $G$-equivariant map, where we think of the input $X\in \R^{n\times d}$ as a row-stack $x\in \R^{nd}$. The works mentioned above assert that $L_{st}=L_{kl}$  if and only if there exists an element $g\in G$ such that $g(s)=l,g(t)=k$. Namely, the indices $(s,t)$ and $(k,l)$ share a parameter if and only if they belong to the same orbit of $G$ when acting on $\set{1,\dots,nd}^2$.

We  now find this parameter-sharing scheme for $G$. For readability, we use two indices $(i,j)$ to represent an index in $s\in \set{1,\dots,nd}$. Given two such indices $(s,t)=(i_s,j_s,i_t,j_t)$ we wish to find their orbit under the action of $G$. We split this question into two cases and treat them one by one: (1) We first consider the case where $i_s\neq j_s$. In this case, the orbit of $(i_s,j_s,i_t,j_t)$ consists of all indices $(l,k)=(i_l,j_l,i_k,j_k)$ such that $i_l \neq i_k$ which, in turn, implies that all the elements of $L$ that are not on the $d\times d$ block diagonal share their parameter. (2) In the case where $i_s=i_t$, applying the group action shows that all the $d\times d$ diagonal blocks represent the same $H$-equivariant function.
\end{proof}
Figure \ref{fig:param_share} illustrates parameter sharing schemes for $G$-equivariant layers for (a) $G=S_n\times H$ and (b) $G=H^n\rtimes S_n$, where $d=4,n=5$ and $H=C_4$ the cyclic group of four elements. Here, each color represents a parameter. Note that all off-diagonal elements in (b) are represented by the same parameter in contrast to (a). The invariant universality proof of Theorem \ref{thm:invariant_universality} applies in this case as well.

\section{Proofs}\label{s:AppendixProofs}
\subsection{Proof of Lemma \ref{l:equi_dim}}
\begin{proof}
As discussed in Section \ref{s:prem}, $L$ can be realized as a $\ell \times \ell$ matrix, and the problem of finding all linear $G$-equivariant functions $L$ can be reduced to solving the following \emph{fixed-point equation}: $g\cdot L = L$. Recall that we are interested in the dimension of the space of linear $G$-equivariant layers, or equivalently, the dimension of the null space of the fixed-point equation. One way to obtain it, is by applying the trace function to the projection operator onto this null-space. See \citep{fulton2013representation}, section 2.2 for a derivation. In our case, this projection  is given by $\phi = \frac{1}{|G|}\sum_{g\in G} P(g) \otimes P(g)$ which implies:
\begin{align*}
E(G) =\trace(\phi) &=\frac{1}{|G|}\sum_{g \in G} \trace(P(g) \otimes P(g))\\
&=\frac{1}{|G|}\sum_{g\in G}\trace(P(g))^2, 
\end{align*}
where $\otimes$ is a Kronecker product, $P(g)\otimes P(g)$ is the matrix representation of the action of $G$ on $\R^{\ell\times \ell}$ and we use the fact that the trace is multiplicative with respect to the Kronecker product.
\end{proof} 
\subsection{Proof of Theorem \ref{thm:invariant_universality}}
\begin{proof}[Proof of lemma \ref{l:encoding}]
This lemma is a generalization of Proposition 1 in \cite{maron2019provably} and we follow their proof idea. By Noether's theorem (see, \eg, \cite{yarotsky2018universal,maron2019universality}), there is a finite set of invariant polynomials $\left(p_i:\R^d\too\R\right)_{i=1}^l$ that generate the ring of invariant polynomials, that is, any invariant polynomial $p(x)$ can be written as $p(x)=q(\left(p_i(x)\}_{i=1}^l\right)$ where $q:\R^l\too \R $ is some general polynomial. We define $u(x)=\left(p_i(x)\right)_{i=1}^l$. 
On one hand, assume that $y=g\cdot x$ then by the invariance of the polynomials $p_i$ we get  $u(y)=u(g\cdot x)=u(x)$. On the other hand, if $u(x)=u(y)$ assume towards contradiction that $g\cdot y\neq x$ for all $g\in G$, then the orbits $G\cdot x$, $G\cdot y$ are disjoint. As both sets are finite, there is a continuous function $f:\R^d\too \R$ such that $f|_{G\cdot x}\leq -2$ and $f|_{G\cdot y}\geq 2$. Using the Stone-Weierstrass theorem \cite{simmons1963introduction} we can get a polynomial $p$ with the property $p|_{G\cdot x}\leq -1$ and $p|_{G\cdot y}\geq 1$. Define $\overline{p} = \frac{1}{|G|}\sum_{g\in G} p(g\cdot x)$ then $\overline{p}$ is a $G$-invariant polynomial and using the discussion above we can write $\overline{p}(x)=q(\left(p_i(x)\}_{i=1}^l\right)$ for some polynomial $q$. This, in turn, implies the following contradiction: 
\[1\leq \overline{p}(y)=q(u(y))=q(u(x))=\overline{p}(x)\leq -1 \]
\end{proof}

\revision{Before we can continue, we need to define the low-dimensional set $\mathcal{E}$ from the theorem statement. There are several possible definitions for $\mathcal{E}$. Here we define it in a relatively simple way. $\mathcal{E}=\cup_{j>k=1}^d \{\sum_{i=1}^n x_{ik} =\sum_{i=1}^n x_{ij}  \}$. Using this definition, $\mathcal{E}$ is finite union of linear sub spaces of codimension 1. The definition makes sure that for every $X\in \R^{n\times d}\setminus \mathcal{E}$, the vector $\sum_{i=1}^n x_i\in \R^d$ has exactly $d$ different numbers. See discussion after the proof for an alternative definition of $\mathcal{E}$ with higher codimension. We note that it is possible to add random noise to break ties and move data points away from $\mathcal{E}$.}

\begin{proof}[Proof of Theorem \ref{thm:invariant_universality}]


As previously mentioned, our first task is \revision{to concatenate to each element the sum of all elements. Next, we encode each element $\widehat{x_i}:=[x_i,\sum_{j=1}^n x_j]$  with a unique $H$-invariant polynomial descriptor $u_H$ that exists according to lemma \ref{l:multi_channel_encoding}. We define the following map $U_H:\R^{n\times d\times 2} \too \R^{n\times d \times l_H}$ by applying $u_H$ in the following way:  
$$U_H(X)_{i,j,:}=u_H(\widehat{x_i}),\quad j=1,\dots,d$$
In other words, $U_H$ encodes each $\widehat{x_i}$ using $u_H$ and repeats this encoding $d$ times on the second dimension of the output tensor. Note that since $u_H$ is $H$-invariant then each component of $U_H$ that is applied to a specific element $x_i$ is $H$-equivariant. 
}
Let $Y\in \R^{n\times d \times l_H}$ denote the output of $U_{H}$ and $y_i=u_H(x_i)$ the unique $H$-invariant descriptors. Our second step is to map this set of unique descriptors to a unique set descriptor. By using Lemma \ref{l:encoding}, there exists a function $u_{S_n}:\R^{n\times l_H}\too\R^{l_{S_n}}$ that computes this encoding. Moreover, in the specific case of $S_n$, $u_{S_n}$ can be chosen to be in the following form: $u_{S_n}(y_1,\dots,y_n)=\sum_{i=1}^n p(y_i)$ where $p:\R^{l_H}\too \R^{l_{S_n}}$ is a multivariate polynomial (see section 4 in \cite{maron2019provably} for more details). We define: 
$$U_{S_n}(Y)= \frac{1}{d}\sum_{i=1}^n\sum_{j=1}^d p(Y_{i,j,:})$$

and note that $U_{S_n}(Y)=u_{S_n}(y_1,\dots,y_n)$ is exactly the unique $S_n$-invariant set descriptor, and that $U_{S_n}$ is a $G$-invariant function as it is composed of applying feature-wise polynomials and summation. \revision{Moreover, according to Lemma \ref{l:G_encoding} the mapping $u$: $X\mapsto U_{S_n}(Y)$ is an injective mapping from $G$ orbits on $K$ to $\R^{l_{S_n}}$ and we can think about it as a function that assigns a unique descriptor to each $G$-orbit}.

Up until now, we have mapped our set of elements $X$ to a unique set descriptor $U_{S_n}(U_H(X))$. Our next step is to map each such set descriptor to the value $f(X)$. Intuitively, we would have liked to apply the function $r=f\circ(U_{S_n}\circ U_H)^{-1}$ to the output of $U_{S_n}\circ U_H$ but unfortunately, $U_{S_n}\circ U_H$ is not injective so an inverse function is not well defined. Because $f$ and $u=U_{S_n}\circ U_H$ are invariant to the action of $G=S_n\times H$ there exists unique continuous maps $\tilde{f},\tilde{u}$ from the quotient space $\mathbb{R}^{n\times d}/G$ such that $f=\tilde{f}\circ \pi$ and $u=\tilde{u}\circ \pi$ where $\pi$ is the projection map to the quotient space. From the fact that our domain $K\subset \mathbb{R}^{n\times d}$ is compact we get that $\tilde{K}=\pi(K)$ is compact and $\tilde{u}$ is bijective between $\tilde{K}$ and its image. We can now write $f=(\tilde{f}\circ \tilde{u}^{-1})\circ \tilde{u} \circ \pi $ and define $r=\tilde{f}\circ \tilde{u}^{-1}$, which is continuous from lemma \ref{l:compo_cont}\\

In the last stage of the proof, we use the universal approximation properties of MLPs \cite{cybenko1989approximation,hornik1989multilayer} in order to approximate the three functions mentioned above, \ie, $U_{S_n},U_H,r$, using a $G$-invariant network. 

\revision{First, calculating the set descriptor $\sum_{j=1}^n x_j$ can be calculated with a single $G$-equivariant layer. Next, we turn to } $U_H$ which is defined as an element-wise application of the $H$-invariant function $u_H$. We note that $U$ applies a continuous $H$-invariant function element-wise which can be approximated by an $H$-invariant network according to our assumption. Furthermore, an element-wise application of an $H$-invariant network is a $G$-equivariant network which implies that there is a $G$-equivariant network $N_H:\R^{n\times 2d}\too \R^{n\times d\times l_H}$ that uniformly approximates it. 

Next, we would like to approximate $U_{S_n}$. From the universality of MLPs  there exists MLP an $M_1:\R^l_H\too \R^{l_{S_n}}$ and such that $M_1$ approximates $p$, which implies that  $\sum_{i=1}^n M_1(y_i)$ approximates $u_{S_n}(\{ y_i\}_{i=1}^n)$. We define the next equivariant layers to apply $M_1$ to the feature dimension of $Y$. We then apply a scaled $G$-invariant summation function in order to get $\frac{1}{d}\sum_{i=1}^n\sum_{j=1}^{d} M_1(y_i)$ as output. Our last function to approximate is $r$ and since it is a continuous function defined on a compact domain we can approximate it with an MLP $M_2$. 

To summarize, we have written our function of interest $f$ as a composition of three functions $U_H,U_{S_n},r$, and constructed a networks that uniformly approximates each one of these functions, which, by using the uniform continuity of the functions, gives us a uniform approximation of their composition.  
\end{proof}
\revision{
\begin{lemma}\label{l:multi_channel_encoding}
Let $H\leq S_d$ act on $\R^{d\times k}$ by applying the same element $h\in H$ to each channel,  then there exists a polynomial function $u:\R^{d\times k}\too \R^l$, for some $l\in \mathbb{N}$, for which $u(x)=u(y)$ if and only if $x=h\cdot y$ for some $h\in H$. 
\end{lemma}
\begin{proof}
We look at an isomorphic copy of $H$ as a subgroup of $S_{d\cdot k}$. From Lemma \ref{l:encoding} there exists such a polynomial descriptor, which is clearly $H$-invariant. 
\end{proof}

\begin{lemma}\label{l:G_encoding}
Let $K\subset \R^{n\times d}$ such that $K\cap \mathcal{E}=\emptyset$ then the polynomial mapping $u$: $X\mapsto U_{S_n}(Y)$ is $G$-invariant and injective as function from $\R^{n\times d}/G$.
\end{lemma}
\begin{proof}
 The map is clearly invariant as a composition of invariant and equivariant functions, so we need show that this encoding is unique modulo $G$ for $X\in \R^{n\times d}$. To see this we first note that since the set encoding in unique up to permutation we can reconstruct $z_{\sigma(1)},...,z_\sigma(n)$ for some $\sigma\in S_n$. Then for each $z_i$ we can reconstrcut $[x_i,y]$ up to a permutation $h_i\in H$, i.e. we can recover 
 $$[h_{\sigma(1)}x_{\sigma(1)},h_{\sigma(1)}y],...,[h_{\sigma(n)}x_{\sigma(n)},h_{\sigma(n)}y]$$
 
 We know that the concatenated $y$ part of the input is identical, and contains only unique numbers from our assumption $K\cap\mathcal{E}=\emptyset$ so we can permute each one such that they all agree with the first element $h_{\sigma(1)}y$. Since all elements of $y$ are unique, the single permutation for the i'th element that does this is $h_{\sigma(1)}h_{\sigma(i)}^{-1}$ to get $[h_{\sigma(1)}x_{\sigma(1)},h_{\sigma(1)}y],[h_{\sigma(1)}x_{\sigma(2)},h_{\sigma(1)}y],...,[h_{\sigma(1)}x_{\sigma(n)},h_{\sigma(1)}y]$ which is the same as our input up to $(\sigma, h_{\sigma(1)})\in G$.\\
 
\end{proof}

}

\begin{lemma}\label{l:compo_cont}
Let $K\subset \R^m$ be a compact domain and $f:K \too \R $ be a continuous function such that $f=h\circ g$. If $g$ is continuous, then $h$ is continuous on $g(K)$.  
\end{lemma}
\begin{proof}
Assume that this is incorrect, then there is a sequence $y_i=g(x_i)$ such that $y_i \rightarrow y_0 $ but $h(y_i)\nrightarrow{} h(y_0)$. Without loss of generality, assume that $x_i\rightarrow x_0\in K$ (otherwise choose a converging sub sequence). We have 
$$f(x_i)=h(g(x_i))=h(y_i)\nrightarrow{}h(y_0)=h(g(x_0))=f(x_0)$$ 
which is a contradiction to the continuity of $f$.
\end{proof}

\revision{
\paragraph{Alternative choices of $\mathcal{E}$.} We can make  $\mathcal{E}$ smaller by using power-sum set descriptors in addition to the sum we used in the proof above, namely $\widehat{x_i}:=[x_i,\sum_{i=1}^n x_i,\sum_{i=1}^n x^2_i,\dots,\sum_{i=1}^n x^n_i]$ where for $x\in \R^d$, $x^k$ is the element-wise power operation. In this case, it is enough for us to have a single $l\in \{1,\dots,n\}$ for which $\sum_{i=1}^n x^k_i$ has $d$ different numbers. Here, we can define $\mathcal{E}=\cap_{l=1}^n\cup_{j>k=1}^d \{\sum_{i=1}^n x^l_{ik} =\sum_{i=1}^n x^l_{ij}  \}$. This is an intersection of $n$ sets of codimension 1.
}

\subsection{Proof of Theorem \ref{thm:equivariant_universality}}
\revision{For the equivariance proof, we need to use Theorem \ref{thm:invariant_universality} on subsets of size $n-1$ of the inputs. For that reason, we define $\mathcal{E}$ to make sure that all sums of such subsets have different numbers. We define: $\mathcal{E}=\cup_{l=1}^n\cup_{j>k=1}^d \{\sum_{i\neq l}^n x_{ik} =\sum_{i\neq l}^n x_{ij}  \}$ . As in the proof of Theorem \ref{thm:invariant_universality}, we can define $\mathcal{E}$ using the power-sum polynomials and get a smaller set.}
\begin{proof}[Proof of Theorem \ref{thm:equivariant_universality}]
We first note that $G$-equivariant polynomials are dense in the space of continuous $G$-equivariant functions over a compact domain. For proof, see Lemma 4 in \cite{segol2019universal}: while the statement in the paper is about $S_n$ equivariant polynomial, the proof trivially extends for every finite group.
We therefore start by approximating  $P:\mathbb{R}^{n\times d}\rightarrow \mathbb{R}^{n\times d}$, an equivariant polynomial map of degree at most $m$. We look at $P_1=\mathbb{R}^{n\times d}\rightarrow \mathbb{R}^{d}$ the first element in the output of $P$. If $P$ is $S_n\times H $ equivariant then by lemma \ref{lemma:s_n-1}  $P_1$ is $S_{n-1}$ invariant when $S_{n-1}$ operates on the last $n-1$ rows and $H$ equivariant. If we fix $x_2,...,x_n$ then $P_1(x_1,...,x_n)$ is a H-equivariant polynomial in $x_1$. The space of H-equivariant polynomials of bounded degree is a finite dimensional linear space and therefore has a basis $q_1,...,q_T$.
We can therefore write $P_1(x_1,...,x_n)=\sum \alpha_k(x_2,...,x_n)\cdot q_k(x_1) $ where $\alpha_k:\mathbb{R}^{n-1\times d}\rightarrow \mathbb{R}$ are the coefficients. Because $P_1$ is $S_{n-1}$-invariant, $H$-equivariant and $q_k$ are a basis it is easy to see that $\alpha_k$ must be $S_{n-1}\times H$ -invariant: If $\sigma$ is a permutation on the last $n-1$ elements then $P_1(x_1,...,x_n)=P_1(x_1,x_{\sigma(2)}...,x_{\sigma(n)})$ since $P_1$ is invariant to $\sigma$. We then have $\sum \alpha_k(x_2,...,x_n)\cdot q_k(x_1)=\sum \alpha_k(x_{\sigma(2)},...,x_{\sigma(n)})\cdot q_k(x_1) $ and because $q_k$ form a basis this means that for each $k$, $\alpha_k(x_2,...,x_n)$ is equal to $\alpha_k(x_{\sigma(2)},...,x_{\sigma(n)})$ proving $S_{n-1}$ invariance. The same idea shows $H$-invariance.

Next, we note that since $P$ is $G$-equivariant we have $[P(x_1,..,x_n)]_i=\sum_k  \alpha_k(x_1,...x_{i-1},x_{i+1},x_n)\cdot q_k(x_i)$  by applying a permutation that only switches 1 and $i$. We will now show how this can be approximated using our $G$-equivariant network. This is the key for the proof as we break down the equivariant function to invariant functions that we can already approximate (up to the fact that they are $S_{n-1}$-invariant and not $S_n$ invariant), and $H$-equivariant functions that can be approximated by the assumption on $H$. The fact that $\alpha_k$ are invariant to permutations of the other $n-1$ elements and not the whole set is not an issue, as that can be implemented easily in our framework as our basic layers separates the sum over other elements with the operation over the current one  (see theorem \ref{thm:uniform}) which is exactly the operation needed.

We will use function names from the proof of theorem \ref{thm:invariant_universality} when applicable for clarity. The first step of approximating $P$ will be $U_H:\mathbb{R}^{n\times d} \too\mathbb{R}^{n\times d\times l_H+1}$ that maps each element to a unique $H$-invariant descriptor (same as in the proof of Theorem \ref{thm:invariant_universality}) plus the original information on a separate channel. The second mapping is $U_{S_{n-1}}:\mathbb{R}^{n\times d\times l_H+1}\rightarrow \mathbb{R}^{n\times d\times l_{S_{n-1}}+1}$ that computes an $S_{n-1}\times H$-invariant representation of the other $n-1$ inputs at each point plus the original input. Next, we need to compute the equivariant polynomial base elements and invariant coefficients. The coefficients are a continuous mapping $r:\mathbb{R}^{l_{S_{n-1}}}\rightarrow \mathbb{R}^{T}$ (proof of continuity is the same as in the proof of Theorem \ref{thm:invariant_universality}) and can be approximated by an MLP, the equivariant polynomials $q_k:\mathbb{R}^d\rightarrow \mathbb{R}^d$ are $H$-equivariant and continuous and can be approximated by an $H$-equivariant network that is applied to each element independently. The last operation is the multiplication and summation over basis elements, which can be approximated by an MLP on the channel dimension.
This breaks down the computation of $P$ into parts that each can be approximated by an equivariant neural network and therefore so can $P$.  Since all polynomials are dense in the space of equivariant functions this shows that each equivariant function can be approximated by an equivariant neural network.
\end{proof}
\begin{lemma}\label{lemma:s_n-1}
If $f:\mathbb{R}^{n\times d}\rightarrow \mathbb{R}^{n\times d}$ is $S_n\times H$ equivariant, then $f_1:\mathbb{R}^{n\times d}\rightarrow \mathbb{R}^{d}$ the first element of $f$ is $S_{n-1}$ invariant and H equivariant. We assume $S_{n-1}$ acts by permuting the last n-1 elements.
\end{lemma}
\begin{proof}
The proof is simple, we can think of $f_1$ as $\pi_1\circ f$, ,\ie, $f$ followed by the projection map on the first element. Since the permutations in $S_{n-1}$ leave the first element in place, $\pi_1$ is invariant to them and so is $f_1$ as composition of equivariant and invariant. It is also clear that $\pi_1$ is H equivariant making $f_1$ H equivariant.

\end{proof}

\section{Implementation details}\label{s:implementation_details}
All experiments (unless stated otherwise) were conducted using the PyTorch framework \cite{paszke2017automatic}, trained with the Adam optimizer \cite{kingma2014adam} on NVIDIA V100 GPU. We performed hyper-parameter search for all methods to choose a learning rate in $\{10^{-1},10^{-2},\dots,{10^-7}\}$. All model architectures use batch normalization \cite{ioffe2015batch} after each linear layer.

\paragraph{Datasets.} The following datasets were used:(1) Places \cite{zhou2017places}, an image dataset with natural scenes such as beach, parking lot or soccer field;
(2) UCF101 \cite{soomro2012ucf101}, an action recognition dataset for realistic action videos; (3) Celeba \cite{liu2018large}, a large scale image dataset that contains celebrity faces; (4) Dynamic Faust \cite{dfaust:CVPR:2017}, triangular meshes of real people performing different activities; (5) ImageNet \cite{deng2009imagenet} large image classification dataset. 

\subsection{Signal classification experiment}
\paragraph{Data preparation.} We generated 30,000 training examples and 3,000 test and validation examples. The type of the signal was uniformly sampled from the three possible types (sine, rectangular and saw-tooth). Frequency and amplitude were uniformly sampled from $[1,10]$, horizontal shift was uniformly sampled from $[0,2\pi]$, vertical shift was sampled from $[-5,5]$. 
From each clean signal, we generate a set of size 25 by replicating the signal and adding independent noise to each copy. The noise is sampled from an i.i.d. Gaussian distribution with zero mean and $3a$  standard deviation where $a$ is the amplitude of the signal.

\paragraph{Network and training.} For training we used batch size of 64 and ran for 200 epochs with validation-based early stopping. Training took between 15 minutes for MLP to 5 hours for DSS(Aitalla). For all layer types, we used three layers followed by a fully connected layer with the following number of features:  MLP $(840, 420, 420)$, Siamese $(220, 220, 110)$, DSS $(160, 160, 80)$, DSS (max) $(160, 160, 80)$,Siamese+DS(2 Siamese layers + a single DS layer) $(200, 200, 100)$, DS $(1000, 1000, 500)$, DS (max) $(1000, 1000, 500)$, Aittala \cite{aittala2018burst} $(160, 160, 80)$, Sridhar \cite{sridhar2019multiview} $(220, 220, 110)$. In models that use convolution, we used strided-convolution with stride $2$. For all models, we have used sum-pooling on the set  and spatial dimensions before the fully connected layer.

\subsection{Image selection}
\paragraph{Data preparation.} The data for the video frame ordering experiment was taken from the UFC101 dataset  \cite{soomro2012ucf101}. For the highest quality image selection task we used the Places dataset \cite{zhou2017places}. For the places dataset, we first selected 25 classes that have the largest number of images. We then generated the train and validation sets from the standard train split, and used the standard validation split as test. In both cases, we used 20,000 training examples and 2,000 validation and test examples. The set sizes are $n=8$ for the frame ordering experiment and $n=20$ for the image quality assessment experiment. Train Image size was reduced to $80\times 80$ and we used random cropping to $64\times 64$ as well as random flipping as data augmentation. For the image quality task we also used random rotations. For the highest image quality task, we sampled a base blur $\sigma\sim U[0,1]$ for each image, and another example specific $\sigma'\sim U[0,1]$ and used  $\sigma+\sigma'$ as the Gaussian width for blurring; i.i.d Gaussian noise was added to the result for the Gaussian noise case and i.i.d random pixels where zeroed-out in the Occlusion noise case.

\paragraph{Network and training.} For training we used batch size of 16 for 200 epochs with validation-based early stopping. training time was 6.5 (3.75) hours for DS and 4.5 (3.25) hours for DSS for the image quality assessment task (video frame ordering task). The network architecture is based on the image anomaly detection network suggested by \cite{zaheer2017deep} and is composed of convolutional part followed by a DeepSets block. The convolutional part consists of three blocks, each of which consists of the following number of features (32,32,64),(64,64,128),(128,128,256) for \textit{DSS(sum)} and \textit{DSS(max)},(90,90,100),(100,100,100)(110,110,128) for \textit{DSS(Aittala)} and (50,50,100),(100,100,180),(200,200,256) for DS+Siamese and DSS(Sridhar). All DeepSets blocks have three layers with features(256,128,1).

\subsection{Shape selection}
\paragraph{Data preparation.}
The data for the shape selection task was taken from Dynamic Fuast \cite{dfaust:CVPR:2017}. This dataset contains 3D videos of 10 human subjects (male and female) performing 15 activities (e.g. jumping jacks, punching, etc.) The data are represented as triangular meshes of the same topology. Importantly, all the shapes are in one-to-one correspondence. For the graph modality we directly use the mesh, see sec. \ref{s:examples_and_generalizations} for more details. For the point-cloud modality we simply use the mesh vertices. We generated sets of 7 frames by randomly cropping sequences and shuffling their order. Note that, since the scans were captured at 60fps, the motion between few consecutive frames is approximately rigid and at constant velocity. To make the problem more challenging we chose to skip every $k$ frames. To choose $k$ we ran the following simple experiment. For each value of $k$ we computed the mean shape of the set by averaging the point coordinates of all the set elements. We then searched for the shape in the set that was closest to the mean shape and evaluated the accuracy on the validation set. The results were $80.95, 43.75, 32.91, 27.73$ for skip sizes of $1, 3, 5, 10$ respectively. We ended up choosing a skip size of 5. For testing we used a held out set. We chose a very challenging split where both the subjects and the activities are not seen at train time.

\paragraph{Network and training.}
We repeated all experiments with 3 different seeds, trained for 100 epochs using Adam optimizer on an NVIDIA TitanX with validation-based early stopping. For the point-cloud modality, all methods were ran using a batch-size of 16. Training times were roughly 2 hours. The network architecture is based on a PointNet module \cite{qi2017pointnet} with 1D convolutions of dimensions (64, 256, 128) followed by a DeepSets block of dimensiones (128, 128, 128, 1). For the graph experiment we used batch-sizes of 8 for the architecture of Aittala, and 12 for all the other architectures. Training took about 20 hours. The architecture is based on a pytorch-geometric \cite{fey2019fast} implementation of Graph Convolutional Networks (GCN) \cite{kipf} with the adjacency matrix: $\hat{A} = A + 2I$. Dimensions of the graph layers were the same as described above for PointNet. We note that CGN, and in general message passing networks on graphs are not universal approximators.

\subsection{Color matching}
\paragraph{Data preparation.}
We used the Places \cite{zhou2017places} and the CelebA \cite{liu2018large} datasets. For the places dataset, we first selected 25 classes that have the largest number of images. We generated the train and validation sets from the standard train split, and used the standard validation split as test. For the CelebA dataset, we used the standard splits. In both cases we generate 30,000 train examples and 3,000 examples for validation and test with resolution $64\times 64$ 

\paragraph{Network and training.} We used U-net like networks. All architecture are composed of an encoder followed by a DeepSets block and a decoder. The encoder and decoder are composed of convolution blocks (2X(conv,batchnorm,relu)) according to the DSS variant/Siamese+DS architecture with each folowed by a max pooling layer with stride=2 for the first encoding layers and stride=8 for the last encoding layer. The DeepSets block is composed of three DeepSets layers with the same number of features as in its input. each decoding block applies similar convolution blocks, upsamples the signal and concatenates the appropriate features from the encoding phase. We use the following number of features: (50,100,150,200) for \textit{DSS(sum)} and \textit{DSS(max)}, (64,128,200,300) for \textit{DSS(Sridhar)} and Siamese+DS and (75,100,150,160) for \textit{DSS(Aittala)}. Training was done with batch size = 32 for 50 epochs starting with initial learning rate 0.001 and learning rate decay of $0.4$ every 10 epochs, and with validation-based early stopping. We use the $L_1$ loss.

\subsection{Burst image deblurring}
\paragraph{Data generation.} We follow the protocol in \cite{aittala2018burst}. We generate blurred images by randomizing a (non-centered) blur kernel and noise. We use a loss that penalizes the deviations of the output image and its gradients from the original image. We also added the mean absolute error of the trivial predictor that outputs the median pixel of the images in the burst at each pixel (the mean predictor produced worse results). we used training set of size 100,000 and test/validation sets of size 10,000, randomly chosen from the ImageNet dataset \cite{deng2009imagenet}. We down-sample images to $128\times 128$ for efficiency.  Training was done with batch size=32, learning rate of $0.003$ and a decay rate of $0.95$ every epoch for 35 epochs and validation-based early stopping.

\paragraph{Network and training.} we have used the same network architecture as in the color channel matching experiment followed by a set max pooling layer and two additional 2D convolutions. We have used the following number of features: (48,50,100,150,200) for \textit{DSS(sum)} and \textit{DSS(max)}, (48,64,128,200,300) for \textit{DSS(Sridhar)} and \textit{Siamese+DSS} and (75,100,110,125,125) for \textit{DSS(Aittala)}.

\section{Examples}\label{s:examples_and_generalizations}
In this subsection, we discuss how our general results can be used in three specific scenarios: learning sets of images, sets of sets, and sets of graphs.

\paragraph{Learning sets of images.} In this case, we write $d=h\cdot w$ for $h,w\in \mathbb{N}$, that is, each $x_i$ is a vector in $\R^{hw}$, and we $H$ to be the group of $2D$ circular translations. According to Theorem \ref{thm:uniform}, a general linear equivariant layer for this setup can be written as $L(x_i)=L_1(x_i) + L_2\left(\sum_{j\neq i} x_j\right)$. In other words, the layer consists of two different convolutional layers, where the first layer $L_1$ is applied to each image independently and the second layer $L_2$ is applied to the sum of all images. This layer is easy to implement and we make an extensive use of it in the experiment section (section \ref{s:exp}). We note that certain temporal or periodic signals can be handled in a similar fashion. In this case $x_i\in \R^k$ and is the group of $1D$ circular translations.

\paragraph{Learning sets of sets.} Another useful application of our theory is for learning sets of consistently-ordered sets. See Section \ref{s:exp} for an example on sets of point-clouds. Here, we set $d = m\times k$ where each item $x_i$ is an $m \times k$ matrix representing a set of $k$-dimensional points. The group $H$ in this case is $S_m$ which acts by permuting rows. We note that equivariance to this type of group action was first considered by \cite{Hartford2018} for learning interactions between sets. In their paper, \citet{Hartford2018} also characterized the maximal linear equivariant basis (a special case of Theorem \ref{thm:uniform}) which (as they nicely show) can be easily implemented by simple summation operations. 

\paragraph{Learning sets of graphs.} Our layers can also be used to learn sets of graphs for tasks such as graph anomaly detection and graph classification. See Section \ref{s:exp} for an example on graphs that represent 3D shapes. In this case, $d=k^2$ and each item $x_i$ is a $k\times k$ tensor (possibly with another feature dimension) representing the interactions between the $k$ vertices in the graph (\eg, an adjacency of affinity matrix). The group $H$ in this case is $S_k$, acting on $x_i$ by permuting its rows and columns. 

\paragraph{Revisiting Deep Sets} As our final example we note that both the characterization of equivariant layers and the universal approximation results in \cite{zaheer2017deep} are special cases of our theoretical results (Theorems \ref{thm:uniform}, \ref{thm:invariant_universality}) where we set $H=I_d$, that is, the symmetry group of the elements $x_i$ is trivial.

\ignore{
\subsection{Beyond linear layers}\label{s:beyond_linear}
In this section we discuss several generalizations of our framework. We have seen that a linear $G$-equivariant layer is a sum of a linear $H$-invariant function that operates on each element independently, with another linear $H$-invariant layer that operates on the sum of all elements. This can be written in the following form: 

\begin{equation} \label{e:general_model}
\begin{split}
    L(X)_{i}=L_1(x_i) + L_2\left(\sum_{j\neq i}x_j\right) =
    f(x_i) + g(x_1,\dots,x_n)_i
\end{split}
\end{equation}
where  $f(x)=L_1(x)$, $g(x_1,\dots,x_n) = L_2\left(\sum_{j\neq i}x_j\right)$ and $L_1,L_2$ are linear $H$-equivariant functions. We call $f$ the \emph{Siamese part} as it processes each item individually, and $g$ the \emph{aggregation part} as it aggregates information from all other members of the set. This decomposition of a linear $G$-equivariant layer into a sum of a Siamese $H$-equivariant function and an aggregation function suggests a way of constructing more general $G$-equivariant layers. More concretely, we consider the following generalizations:

\paragraph{Non-linear Siamese functions} We can consider non-linear Siamese functions $f$ in equation \ref{e:general_model}. For example in case we want to process sets of graphs, one might want to replace the equivariant linear layers from \cite{maron2018invariant} with non-linear message-passing layers \cite{Gilmer2017} or the quadratic layer suggested in \cite{maron2019provably}.  

\paragraph{Non-linear aggregation functions.} We can consider replacing $g$ with non-linear aggregation functions. Below, we explore a few reasonable options:

The simplest type of aggregation functions are $S_n$-invariant functions, \ie, functions for which the following holds: $g(x_1,\dots,x_n)_i=g(x_1,\dots,x_n)$ for all $i=1,\dots,n$.
In this category we can list the summation function $g(x_1,\dots,x_n) = h\left(\sum_{j=1}^{n}x_j\right)$ where $h$ is any $H$-equivariant function, and the summation can be replaced with other operators such as maximum \cite{aittala2018burst}, median and mean. A possible downside of such aggregation functions is that any item in the set is treated in the same way, which might not be optimal. The second type of aggregation functions are $S_n$-equivariant and take into account the specific instance $x_i$ when aggregating information from the other members of the set. The most popular example of such aggregation function is an attention mechanism \cite{bahdanau2014neural} which calculates different weights for each element. In this case $g$ takes the following form:
\begin{align*}
      g(X)_{i}= \sum_{j=1}^{n}\alpha_{ij}h(x_j), \quad \alpha_{ij}=\frac{\exp(h'(x_i)^T\cdot h''(x_j))}{\sum_{k=1}^n \exp(h'(x_i)^T\cdot h''(x_k))}
\end{align*}
  
where $h,h',h''$ are $H$-equivariant linear functions and we can replace $\alpha$ with any other attention mechanism. We note that similar attention mechanisms were considered in the context of sets in several previous works: \citet{vinyals2015order, ilse2018attention, lee2019set,yang2018attentional} consider attention mechanisms on feature vectors without symmetries while \cite{liu2019permutation} considered separate attention blocks for images. See Section \ref{s:exp} for an experiment with an attention aggregation.}

\section{Multi-view reconstruction}\label{s:multiview}
We tested DSS on a multi-view reconstruction task. Here, the input is a set of images of a 3D object and the task is to predict its 3D structure. We closely follow \citet{sridhar2019multiview}, that pose this task as a learning problem in which a network is trained to ``lift'' image pixels in each view to their 3D normalized coordinate space (NOCS). NOCS is unique in that it canonicalizes shape pose and scale and thus makes the view-aggregation as simple as a union operation. In addition to predicting the NOCS representation of each foreground pixel, the network also predicts the coordinates of the occluded part of the object as if the camera was an x-ray. 

The architecture proposed by \cite{sridhar2019multiview} advocates a mean-subtraction aggregation scheme. After each convolutional block, the mean of all set elements is subtracted. This aggregation scheme can be seen as a specific case of DSS, in which the sum of all elements is further processed by a different convolution layer and only then added to the elements. This raises up an interesting question of whether a simple modification to the architecture of \cite{sridhar2019multiview}, in the form of changing the aggregation step to apply a convolution block to the sum of all set elements, can improve performance. 




Following the same experimental settings prescribed by the authors, we tested the modified architecture (named Sridhar+DSS) in the case of a fixed-sized input of $3$ views per model on three different classes of 3d objects as in \cite{sridhar2019multiview}: cars, Airplanes and chairs. The results are summarized in table \ref{t:multi-view}. As can be seen, our proposed modification gives a significant boost in performance on 2 out of 3 object classes. While we lack a good explanation for why the performance on chairs is decreased, this result suggests that it is worth to further explore the potential benefit of DSS for this task. We leave the full exploration of this specific task to future work. 

\begin{table}[th]
  \centering
  
    \begin{tabular}{l|c|c}
        \hline
         Category & Sridhar & Sridhar+DSS \\
        \hline
        Cars &  0.1645 & \textbf{0.1273} \\ 
        Airplanes & 0.1571 & \textbf{0.1163} \\ 
        Chairs & \textbf{0.1845} & 0.2345 \\ \hline
        Average &  0.1687& \textbf{0.1593}  \\ \hline

    \end{tabular}%
    \caption{Reconstruction error for the Multi-view 3D object reconstruction task. We compare the performance reported in \cite{sridhar2019multiview} and our suggested modification (Sridhar+DSS). Reported errors are 2-way Chamfer distance between the ground truth shape and its reconstruction, multiplied by 100.}        \label{t:multi-view}
\end{table}

\end{document}

%% file: math_commands.tex
\usepackage{amsmath,amsfonts,bm,amsthm}

\newcommand{\set}[1]{\left\{#1\right\}}
\newcommand{\mset}[1]{\left\{\kern-.5em\left\{ #1 \right\}\kern-.5em\right\}}
\newcommand{\mmset}[1]{\{\kern-.4em\{ #1 \}\kern-.4em\}}

\newcommand{\too}{\rightarrow}

\newcommand{\trace}{\textrm{tr}} 

\newcommand{\ones}{\mathbf{1}}

\makeatletter
\newtheorem*{rep@theorem}{\rep@title}
\newcommand{\newreptheorem}[2]{%
\newenvironment{rep#1}[1]{%
 \def\rep@title{#2 \ref{##1}}%
 \begin{rep@theorem}}%
 {\end{rep@theorem}}}
\makeatother

\newtheorem{theorem}{Theorem}
\newreptheorem{theorem}{Theorem}
\newtheorem{lemma}{Lemma}
\newreptheorem{lemma}{Lemma}
\makeatletter
\newcommand{\subalign}[1]{%
  \vcenter{%
    \Let@ \restore@math@cr \default@tag
    \baselineskip\fontdimen10 \scriptfont\tw@
    \advance\baselineskip\fontdimen12 \scriptfont\tw@
    \lineskip\thr@@\fontdimen8 \scriptfont\thr@@
    \lineskiplimit\lineskip
    \ialign{\hfil$\m@th\scriptstyle##$&$\m@th\scriptstyle{}##$\crcr
      #1\crcr
    }%
  }
}
\makeatletter

\newcommand{\eg}{{e.g.}}
\newcommand{\ie}{{i.e.}}
\newcommand{\haggai}[1]{{\color{red}{\bf[Haggai:} #1{\bf]}}}

\newcommand{\revision}[1]{{\color{black}{#1}}}

\def\eqref#1{equation~\ref{#1}}

\def\1{\bm{1}}

\DeclareMathAlphabet{\mathsfit}{\encodingdefault}{\sfdefault}{m}{sl}
\SetMathAlphabet{\mathsfit}{bold}{\encodingdefault}{\sfdefault}{bx}{n}

\newcommand{\R}{\mathbb{R}}

